\documentclass[twoside]{article} 
\usepackage{aistats2017}

\usepackage{multirow}
\usepackage{amsthm,amssymb,amsmath,verbatim,mathrsfs}
\usepackage{graphicx}
\usepackage{color}

\usepackage{caption}
\usepackage{subcaption}
\usepackage[export]{adjustbox}

\usepackage{algpseudocode,algorithm,algorithmicx}

\usepackage[colorlinks]{hyperref}

\floatname{algorithm}{Technique}
\newcommand*\Let[2]{\State #1 $\gets$ #2}
\newcommand*\LetFour[4]{\State #1 $\leftarrow$ #2, #3 $\leftarrow$ #4}
\algrenewcommand\algorithmicrequire{\textbf{Input:}}

\graphicspath{{figures/}}

\newcommand\vecK{\mathbf{k}}
\newcommand\vecM{\mathbf{m}}
\newcommand\vecS{\mathbf{s}}
\newcommand\vecU{\mathbf{u}}

\newcommand\vecX{\mathbf{x}}

\newcommand\vecO{\boldsymbol{\omega}}
\newcommand\vecL{\boldsymbol{\lambda}}

\newcommand\bbE{\mathbb{E}}
\newcommand\bbR{\mathbb{R}}
\newcommand\bbZ{\mathbb{Z}}

\newcommand{\iD}{d} % input dimension
\newcommand{\sS}{n} % sample size
\newcommand{\hRatio}{m} % ratio between grid steps
\newcommand{\sRatio}{\delta} % ratio between sizes of variable fidelity samples

\newtheorem{theorem}{Theorem}
\newtheorem{corollary}{Corollary}
\newtheorem{lemma}{Lemma}
\newtheorem{remark}{Remark}

\definecolor{gold}{RGB}{237,177,32}

\graphicspath{}

\begin{document}



\onecolumn

\aistatstitle{Minimax Error of Interpolation and Optimal Design of Experiments for Variable Fidelity Data}

\aistatsauthor{ A.Zaytsev \And E.Burnaev}

\aistatsaddress{a.zaytsev@skoltech.ru \And e.burnaev@skoltech.ru}
\vspace{-50px}
\begin{center}
Skolkovo Institute of Science and Technology (Skoltech), \\
Insitute for Information Transmission Problems (IITP RAS)
\end{center}
\vspace{50px}
\begin{abstract}
Engineering problems often involve data sources of variable fidelity with different
costs of obtaining an observation. 
In particular, one can use both a cheap low fidelity
function (e.g. a computational experiment with a CFD code) 
and an expensive high fidelity function (e.g. a wind tunnel experiment)
to generate a data sample in order to construct a regression model of a high fidelity function. 
The key question in this setting is
how the sizes of the high and low fidelity data samples should be selected in order
to stay within a given computational budget and maximize accuracy of the regression
model prior to committing resources on data acquisition. 

In this paper we obtain minimax
interpolation errors for single and variable fidelity scenarios for a multivariate
Gaussian process regression. 
Evaluation of the minimax errors allows us to identify
cases when the variable fidelity data provides better interpolation accuracy than
the exclusively high fidelity data for the same computational budget. 

These results allow us to calculate the optimal shares of variable fidelity data samples under the given computational budget constraint. 
Real and synthetic data experiments suggest
that using the obtained optimal shares often outperforms natural heuristics in terms
of the regression accuracy. 
\vspace{130px}
\end{abstract}

\section{INTRODUCTION}

In some cases sample data for regression modeling has variable fidelity: some data
comes from a high fidelity source, some -- from a low fidelity source~\cite{forrester2007multi}.
While there are many approaches to handle variable fidelity data including transfer
learning~\cite{pan2010survey} and space mapping~\cite{bandler2004space} techniques,
engineers often use cokriging approach~\cite{kennedy2000predicting} based on the Gaussian
process framework~\cite{zaytsev2014properties,rasmussen2006}. 
Numerous applications
of cokriging include geostatistics~\cite{xu1992integrating}, aerospace~\cite{han2013improving},
and engineering~\cite{koziel2014efficient}. 
In this paper we also consider this approach for modeling data, 
obtained from high and low fidelity data sources.

The interest in accuracy of Gaussian process models for single fidelity data dates
back to Wiener and Kolmogorov~\cite{wiener1949extrapolation,kolmogorov1992interpolation}.
They obtained an error at a specified point in the univariate case. 
Further progress in refining this estimate is available in the book by Stein~\cite{stein2012interpolation},
inspired by Ibragimov results~\cite{ibragimov2012gaussian}. 
Recent results expand
this setting by considering a more general interpolation error, equal to the integral
of the squared difference between the true function and an its interpolation over
the domain of interest, see~\cite{van2008rates} for finite sample results in the
multivariate case, and~\cite{golubev13interpolation} for results about the minimax
error of interpolation over an infinite regular univariate sample. 
References~\cite{zhang2015divide,suzuki2012pac,bhattacharya2014anisotropic}, to name a few,
report similar results.

While in case of single fidelity data results are quite well established, there is only one
paper~\cite{zhang2015doesn}, to our knowledge, that investigates the interpolation
error for the variable fidelity data case from a theoretical point of view. 
For a squared exponential covariance function and a squared error at a single point, 
authors identify cases when regression modeling based on variable fidelity data is superior
to using only the high fidelity data. 
Other papers dealing with variable fidelity
regression modeling, \cite{daira2016maximum,bevilacqua2015covariance,pascual2006estimation},
focus on statistical properties of regression parameters estimates, but provide
little insight into understanding how and why the variable fidelity modeling works.

Due to current apparent scarcity of theoretical foundations practitioners usually
adopt heuristic rules in determining sizes of data samples of different fidelity
and quantify when to use the variable
fidelity data~\cite{alexandrov1999optimization,simpson2008design,kuya2011multifidelity};
or they use adaptive design of experiments approaches and surrogate based
optimization directly, see~\cite{ranjan2011follow,kandasamy2016gaussian,burnaev2015adaptive,le2015cokriging} and references therein.

The main contributions of this paper are the following:
\begin{itemize}
    \item \textbf{Minimax interpolation error for the multivariate case.}
    We start with obtaining the interpolation error for the Gaussian process regression
    with a known covariance function. 
    Then we derive the minimax interpolation error
    for functions from a general smoothness class in the multivariate case. 
    This error is a nontrivial generalization of the univariate results obtained in~\cite{golubev13interpolation}.

    \item \textbf{The optimal ratio of sizes of variable fidelity data samples.}
    We obtain the interpolation error for the specified covariance function in the
    variable fidelity case, and then derive the minimax interpolation error in
    the general additive setting (cokriging)~\cite{kennedy2000predicting}.
    % --- that is useful as the true covariance function is rarely available. Explain setting
    With the derived minimax interpolation error we identify when and to which extent
    the variable fidelity regression modeling is beneficial compared to the regression
    modeling using only a high fidelity data under the same computational budget. 
    We calculate the optimal ratio of sizes of variable fidelity data samples given the
    budget constraint. 
    There is a certain gap between the theoretical setup we consider
    and the real world: we consider a setting that uses an infinite grid as a design
    of experiments and requires knowledge of relative complexities of high and low
    fidelity functions to calculate the optimal ratio of sample sizes. 
    Nevertheless these theoretical results are sufficient to provide justification for the corresponding applied algorithm we develop.

    \item \textbf{The technique to select the ratio of sizes of variable fidelity data samples.}
    We elaborate on a method to choose the ratio
    inspired by our theoretical investigations. 
    While the existing approaches usually work in adaptive design of experiments setting and pick points using sufficiently accurate regression models constructed beforehand~\cite{ranjan2011follow}, we offer a method to select sizes of high and low fidelity data samples to fit into a given computational budget and maximize accuracy of a resulting regression model
    prior to spending any significant resources on data generation. 
    Our estimate depends only on the computational cost of variable fidelity data generation and
    on a correlation between high and low fidelity functions. 
    We investigate the applicability of the proposed technique by comparing it to a number of natural baselines on synthetic and real datasets.
\end{itemize}
We provide proofs of all theorems in appendicies.

%!TEX root = article.tex

\section{Minimax Interpolation Error for Gaussian Process Regression}
\label{sec:single_fidelity_total}

In case of Gaussian process regression modeling there is a gap between theoretically
tractable problems and practice. 
Namely, since the main tool for theoretical investigation
is the Fourier transform, 
it is a common approach to consider the design of experiments
based on an infinite grid~\cite{golubev13interpolation,stein2012interpolation}, 
though in many cases the theoretical results are transferable to practical solutions. 
Thus, in this section we consider a design of experiments, belonging to some infinite grid, 
and later in the experimental section we show that our conclusions remain valid 
under finite sample random designs.

\subsection{Interpolation Error}
\label{sec:single_fidelity}

Let $f(\vecX)$ be a stationary Gaussian process on $\bbR^{\iD}$ with a covariance
function $R(\vecX) = \bbE (f(\vecX_0 + \vecX) - \bbE f(\vecX_0 + \vecX))(f(\vecX_0) - \bbE f(\vecX_0))$ and a spectral density $F(\vecO)$
\begin{equation*}
    F(\vecO) = \int_{\bbR^{\iD}}
                e^{2 \pi \mathrm{i} \vecO^{\mathrm{T}} \vecX} R(\vecX)
            d\vecX  \,.
\end{equation*}

Suppose that we know values of realizations of $f(\cdot)$ at the infinite rectangular
grid $D_H = \{\vecX_{\vecK}: \vecX_{\vecK} = H \vecK, \vecK \in \mathbb{Z}^{\iD} \}$,
where $H$ is a diagonal matrix with elements $h_1, \ldots, h_{\iD}$. An example of such
design in the case of the input dimension $\iD = 2$ is provided in Figure~\ref{fig:design}.

\begin{figure}
   \centering
   \includegraphics[width=0.4\textwidth]{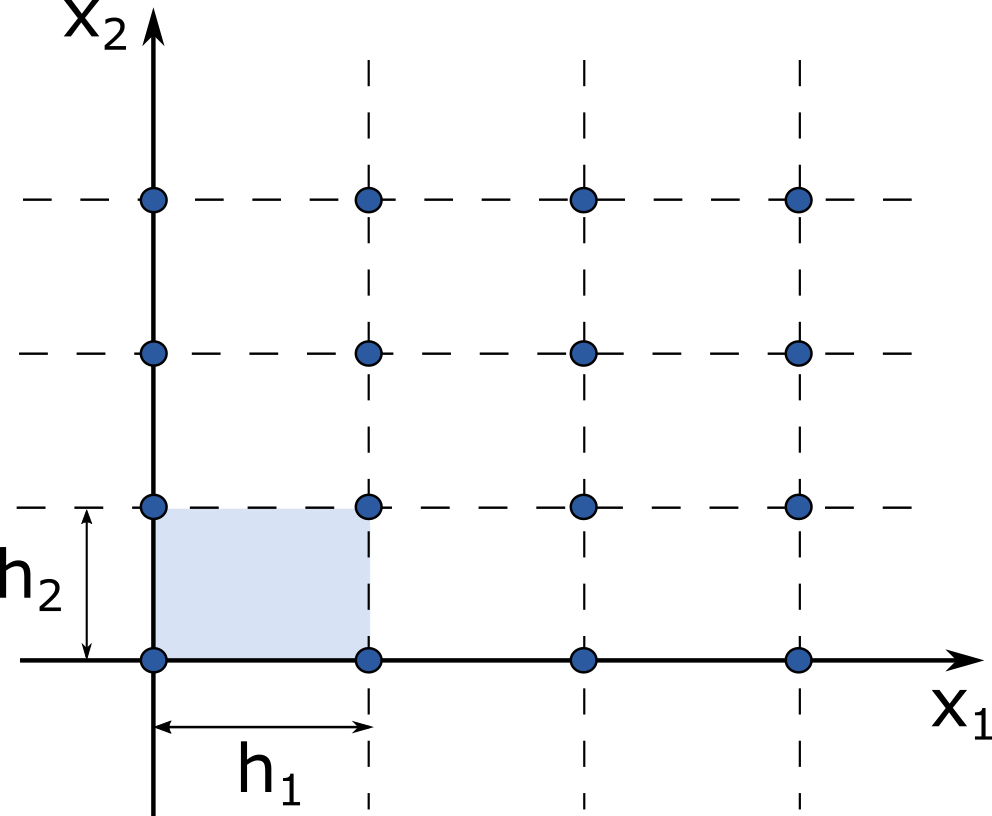}
   \caption{Design $D_H$ for $\iD = 2$.}
   \label{fig:design}
\end{figure}

We measure the interpolation error over the
domain of interest $\Omega_H = [0, h_1] \times \ldots \times [0, h_{\iD}]$ as follows:
\begin{equation} \label{eq:interpolation_error}
    \sigma_H^2(\tilde{f}, F)
        = \frac{1}{\mu(\Omega_H)}
            \int_{\Omega_{H}}
                \bbE \left[\tilde{f}(\vecX) - f(\vecX) \right]^2
            d\vecX \,,
\end{equation}
where $\mu(\Omega_H) = \prod_{i = 1}^{\iD} h_i$ is the Lebesgue measure of $\Omega_H$,
and $\tilde{f}(\vecX)$ is the interpolation of $f(\vecX)$.
Here we consider $\tilde{f}(\vecX)$ having the form 
\begin{equation} \label{eq:tildeF}
    \tilde{f}(\vecX) = \mu(\Omega_H) \sum_{\vecX' \in D_H} K(\vecX - \vecX') f(\vecX_{\vecK}) \,,
\end{equation}
where $K(\cdot)$ is a symmetric kernel.

\begin{theorem} \label{th:interpolation_error}
The error of interpolation with $\tilde{f}(\vecX)$ from~\eqref{eq:tildeF}, based on observations
at points from $D_H$ of a stationary Gaussian process $f(\vecX)$ with spectral density $F(\vecO)$,
is equal to
\begin{align*}
    \sigma_H^2(\tilde{f}, F)
        &= \int_{\bbR^{\iD}}
            F(\vecO) \left[
                \left(1 - \hat{K}(\vecO)\right)^2 + \sum_{\vecX \in D_{H^{-1}} \setminus \{\mathbf{0}\}}
                    \hat{K}^2 \left(\vecO + \vecX \right)
            \right] d\vecO \,,
\end{align*}
where $\hat{K}(\vecO)$ is the Fourier transform of $K(\vecO)$. 
Furthermore, the optimal $\hat{K}(\vecO)$, minimizing the interpolation error, has the form
\begin{equation*}
    \hat{K}(\vecO)
        = \frac{F(\vecO)}{\sum_{\vecX \in D_{H^{-1}}} F \left(\vecO + \vecX \right)}
        \,.
\end{equation*}
\end{theorem}

\begin{remark} \label{lemma:best_approximation}
The function $\tilde{f}(\vecX)$ that minimizes the squared error $\bbE (\tilde{f}(\vecX) - f(\vecX))^2$ has the form~\eqref{eq:tildeF}, where $K(\cdot)$ is a symmetric kernel. 
This motivates us to use $\tilde{f}(\vecX)$ from \eqref{eq:tildeF} for interpolation.
\end{remark}

\begin{remark}
It is easy to see that for $\tilde{f}(\vecX)$ from \eqref{eq:tildeF} it holds that
\begin{equation*}
    \sigma_H^2(\tilde{f}, F) = \sigma_{SH}^2(\tilde{f}, F) \,,
\end{equation*}
where $S = \mathrm{diag}(s_1, \ldots, s_{\iD})$, with $s_i \in \bbZ^{+}, i = 1, \ldots, \iD$.
\end{remark}

Using Theorem~\ref{th:interpolation_error} one can estimate interpolation errors for various
covariance functions. For example,
\begin{corollary} \label{col:exponential_error}
For a Gaussian process on $\bbR$ with exponential spectral density $F_{\theta}(\omega) = \frac{\theta}{\theta^2 + \omega^2}$
the interpolation error~\eqref{eq:interpolation_error} has the form:
\begin{equation*}
    \sigma_h^2(\tilde{f}, F_{\theta}) \approx \frac{2}{3} \pi^2 \theta h + O((\theta h)^2),\,\, 
    \theta h \rightarrow 0 \,.
\end{equation*}
% \begin{align*}
%     &\sigma_h^2(\tilde{f}, F_{\theta}) = \pi - \frac{\pi}{\theta \pi h \coth(\pi \theta h)} + \\
%     & + \frac{\pi h}{2} \left(\cosh(\pi h \theta) - \sinh(\pi h \theta) \right)
%         \frac{\coth^2(\pi \theta h) - 1}{\pi h \coth(\pi \theta h)} \cdot \\
%     & \cdot \left(\pi \cosh(\pi h \theta) - \left(\frac{1}{h \theta} + \pi \right) \sinh(\pi h \theta) \right)
%     \,. 
% \end{align*}
\end{corollary}

\begin{corollary} \label{col:squared_exponential_error}
For a Gaussian process on $\bbR$ with squared exponential spectral density $F_{\theta}(\omega) = \frac{1}{\sqrt{\theta}} \exp \left(-\frac{\omega^2}{2 \theta}\right)$
the interpolation error~\eqref{eq:interpolation_error} is bounded by:
\begin{align*}
    &\frac{4}{3} h \sqrt{\theta} \exp \left(-\frac{1}{8 h^2 \theta}\right)
        \leq \sigma^2_{h}(\tilde{f}, F_{\theta}) 
    \leq 7 h \sqrt{\theta} \exp \left(-\frac{1}{8 h^2 \theta} \right),\,\, \theta h^2 \rightarrow 0
    \,.
\end{align*}
\end{corollary}

\subsection{Minimax Interpolation Error}
\label{subsec:minimax_inter_error}

For many covariance functions direct evaluation of the interpolation error can be
technically cumbersome, especially for the case $\iD > 1$. Furthermore, in many cases
the true covariance function is not known exactly, and calculating the interpolation error in
such misspecified cases is even a harder task.

Instead we consider a minimax interpolation error that provides an answer in the worst
case scenario. We define a set $\mathcal{F}(L, \vecL)$ of spectral densities $F(\vecO)$
for a given $\vecL = (\lambda_1, \ldots, \lambda_{\iD})\in\bbR^{\iD}$ and $L > 0$ as
\begin{equation} \label{eq:Fset}
    \mathcal{F}(L, \vecL)
        = \left\{ F\,:\,
            \bbE \sum_{i = 1}^{\iD} \lambda_i^2 \left( \frac{\partial f_F(\vecX)}{\partial x_i} \right)^2
            \leq L \,, \vecX \in \mathbb{R}^{\iD}
        \right\} \,,
\end{equation}
where $f(\vecX) = f_F(\vecX)$ is a Gaussian process with the spectral density $F(\vecO)$ at the point $\vecX \in \mathbb{R}^{\iD}$. 
% Here and below we use the Euclidean $l_2$ norm $\left\| \vecX \right\| = \sqrt{\sum_{i = 1}^{\iD} x_i^2}$.
% TODO move to supplementary and $\left\| \vecX \right\|_{\infty} = \max_{i \in \{1, \ldots, \iD\}} |x_i|$.
Sample realizations of Gaussian processes for different $L$ in the case of $\iD = 1$
and the Mat\'ern covariance function \cite{rasmussen2006} are shown in Figure~\ref{fig:gp_realizations}.

\begin{figure*}
    \centering
    \begin{subfigure}[b]{0.3\textwidth}
        \includegraphics[height=90px]{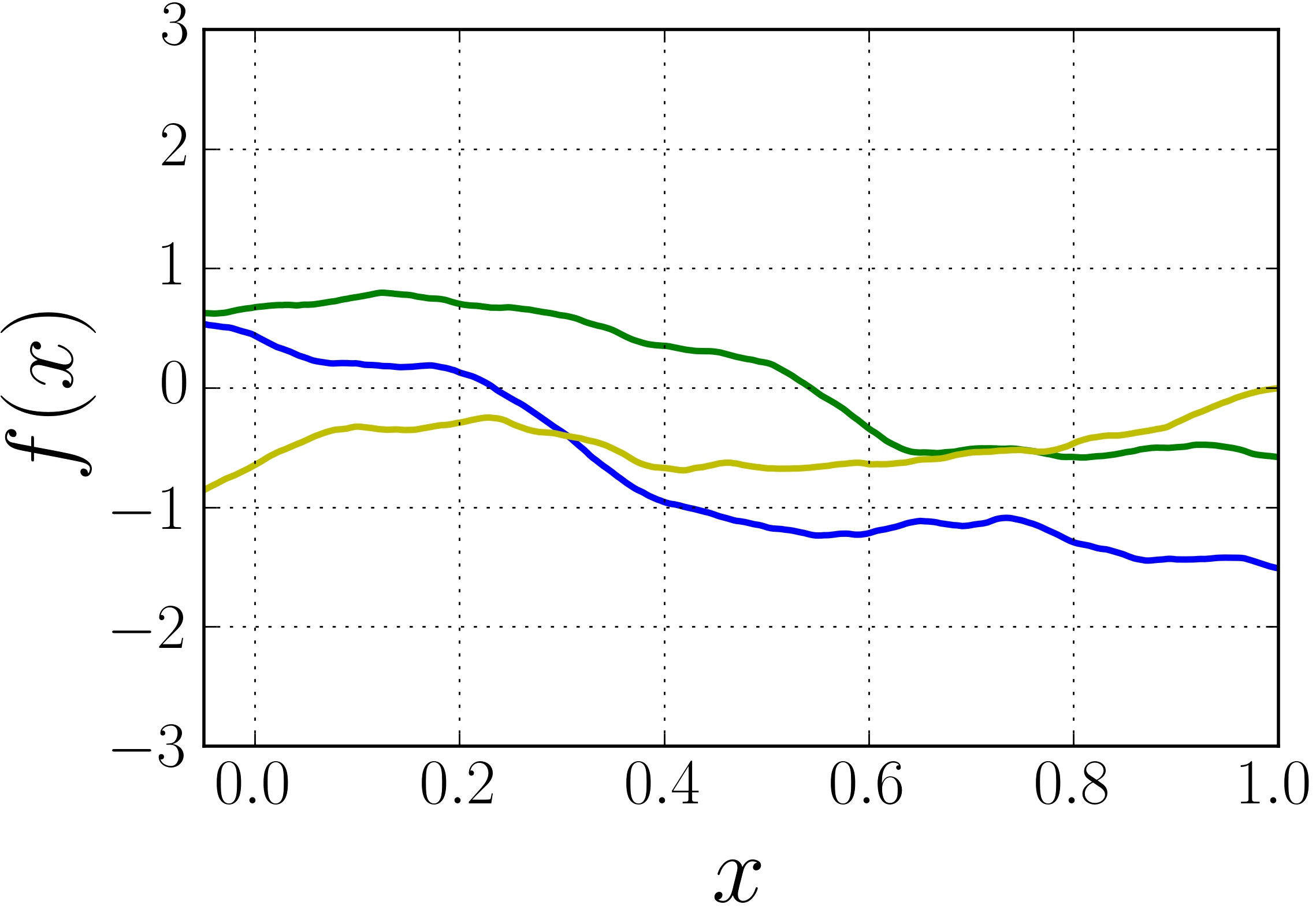}
        \caption{$\mathcal{F}(10, 1)$}       
    \end{subfigure}
    ~
    \begin{subfigure}[b]{0.3\textwidth}
        \includegraphics[height=90px,right]{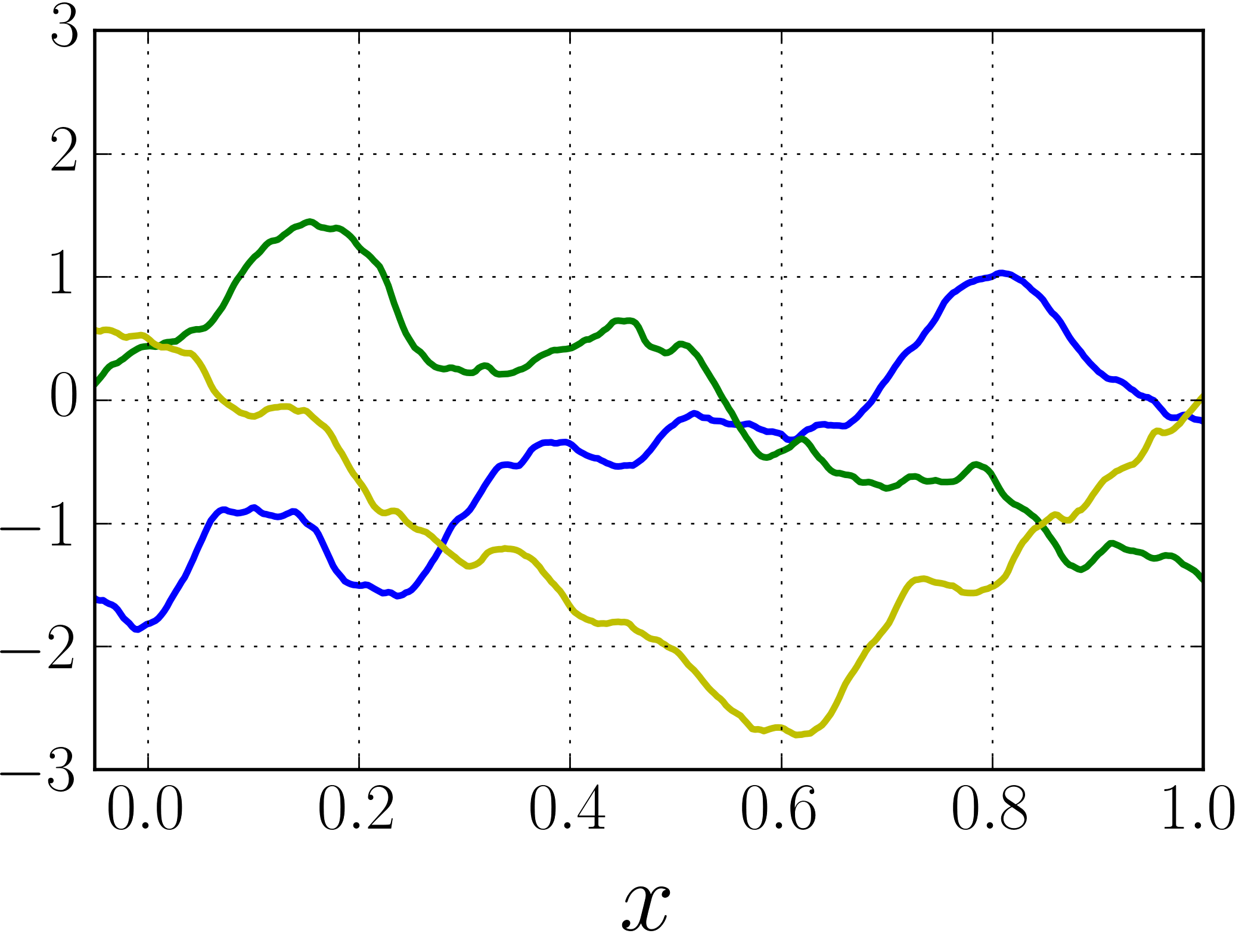}
        \caption{$\mathcal{F}(100, 1)$}
    \end{subfigure}
    ~
    \begin{subfigure}[b]{0.3\textwidth}
        \includegraphics[height=90px,right]{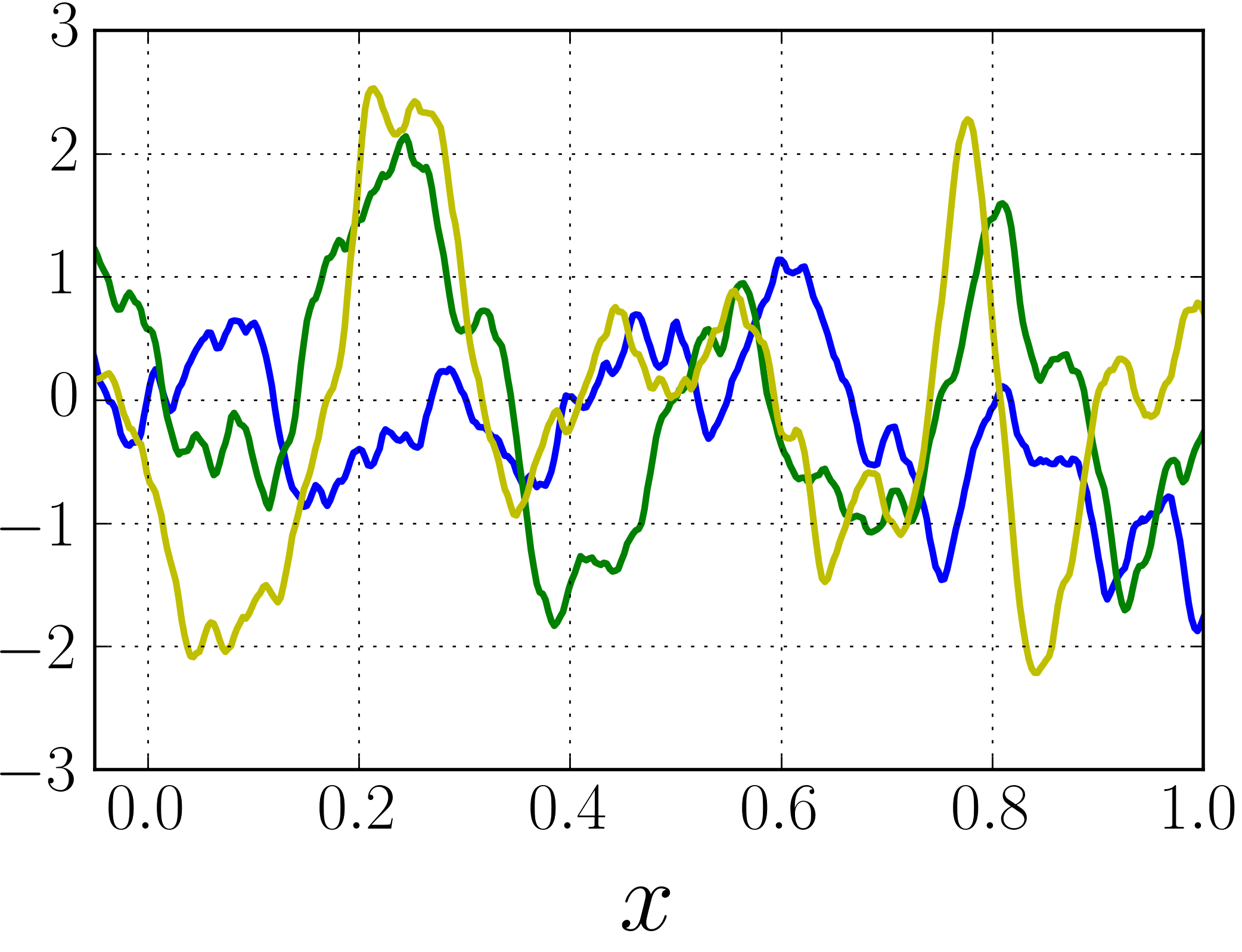}
        \caption{$\mathcal{F}(1000, 1)$}
    \end{subfigure}
    \caption{Realizations of Gaussian processes with the Mat\'ern covariance function
    $R(x) = (1 + \sqrt{3} \theta |x|)\exp(-\sqrt{3}\theta |x|)$ ($\nu=\frac{3}{2}$)
    for different values of $L$ in $\mathcal{F}(L, 1)$ and $\iD = 1$.}
    \label{fig:gp_realizations}
\end{figure*}

% (1 + \sqrt{3} \theta |x - y|) \exp(-\sqrt{3} \theta |x - y|)
% -3 \theta^2 |x - y| \exp(-\sqrt{3} \theta |x - y|)
% 3 \theta^2 (1 - \sqrt{3} \theta |x - y|) \exp(-\sqrt{3} \theta |x - y|)
% 3 \theta^2 = L

The minimax interpolation error is defined as follows:
\begin{equation*}
    R^H(L, \vecL)
        = \inf_{\tilde{f}} \sup_{F \in \mathcal{F}(L, \vecL)}
            \sigma^2_H(\tilde{f}, F) \,.
\end{equation*}
This error describes how large the interpolation error is for the worst case scenario.
The following theorem holds:
\begin{theorem} \label{th:multivariateMinimax}
For a Gaussian process $f(\vecX)$, defined on $\mathbb{R}^{\iD}$ and evaluated on the
design $D_H$, with the spectral density belonging to the set $\mathcal{F}(L, \vecL)$,
the minimax interpolation error has the form
\begin{equation*}
    R^H(L, \vecL)
        = \frac{L}{2 \pi^2}
            \max_{i \in \{1, \ldots, d\}}
                \left(\frac{h_i}{\lambda_i} \right)^2
            \,.
\end{equation*}
Moreover, the minimax optimal interpolation $\tilde{f}(\vecX)$ has the form
\begin{equation*}
    \tilde{f}(\vecX) = \mu(\Omega_H) \sum_{\vecX' \in D_H} K(\vecX - \vecX') f(\vecX') \,, 
\end{equation*}
where $K(\vecX)$ is a symmetric kernel with the Fourier transform $\hat{K}(\vecO)$
defined as
\begin{equation*}
\hat{K}(\vecO) = 
    \begin{cases}
     1 - \sqrt{\sum_{i = 1}^{\iD} \omega_i^2\cdot h_i^2}
        &\text{ if } \sum_{i = 1}^{\iD} \omega_i^2\cdot h_i^2 \leq 1\,, \\ 
     0, &\text{ otherwise} \,.
    \end{cases}
\end{equation*}
\end{theorem}

% As expected we get the minimax interpolation error of the order $O(h^2)$, which is
% smaller than the interpolation error for the exponential covariance function of the
% order $O(h)$ and bigger than the interpolation error for the squared exponential
% covariance function of the order $O(\exp(-c / h^2))$.

Note, that we can minimize the minimax interpolation error w.r.t. the diagonal matrix
$H$ in such a way as to keep fixed the average number of points belonging to a unit
hypercube: $ \prod_{i = 1}^{\iD} \frac{1}{h_i} = n$. The corresponding optimal matrix
$H^* = \mathrm{diag}(h_1^*, \ldots, h_d^*)$ has the form:
\begin{equation*}
    h_i^* = \sqrt[d]{\frac{n \lambda^{\iD}_i}{\prod_{j = 1}^{\iD} \lambda_j}} \,.
\end{equation*}
The minimal minimax interpolation error is then equal to
$
    R^{H^*}(L, \vecL) = \frac{L}{2 \pi^2}
        \sqrt[\frac{d}{2}]{\frac{n}{\prod_{i = 1}^{\iD} \lambda_i}} \,.
$

\section{Minimax Interpolation Error for a Variable Fidelity Model}
\label{sec:minimax_risk}

\subsection{Variable Fidelity Data Model}
\label{sec:model_vfgp}

Suppose that the true function is modelled as
\begin{equation} \label{eq:eqqq}
    u(\vecX) = \rho f(\vecX) + g(\vecX) \,,
\end{equation}
where $\rho$ is a fixed constant, and $f(\vecX)$ and $g(\vecX)$ are stationary independent
Gaussian processes, defined on $\mathbb{R}^{\iD}$. This is the state-of-the-art cokriging
approach used to model a variable fidelity data~\cite{kennedy2000predicting}.

We refer to a realization of $u(\vecX)$ as a high fidelity function, and to a realization
of $f(\vecX)$ as a low fidelity function. 
Therefore $g(\vecX)$ is a correction of $f(\vecX)$
that appears due to a low fidelity nature of $f(\vecX)$. 
The parameter $\rho$ provides
information on a strength of the relation between $f(\vecX)$ and $u(\vecX)$.

We observe values of $u(\vecX)$ and $f(\vecX)$ and we want to construct an interpolation
$\tilde{u}(\vecX)$ of the high fidelity function $u(\vecX)$ on the basis of these variable
fidelity observations.

\subsection{Interpolation Error}
\label{subsec:error_vfgp}

It is natural to assume that we observe the cheap low fidelity function $f(\vecX)$
on denser grid than the expensive high fidelity function $u(\vecX)$. We observe $u(\vecX)$
at points from $D_u = D_H$, and $f(\vecX)$ at points from $D_f = D_{\frac{H}{\hRatio}}$
with $\hRatio \in \mathbb{Z}^{+}$.

Using these observations we attempt to interpolate $u(\vecX)$ within the hypercube
$[0, h_1] \times \ldots \times [0, h_{\iD}]$. The function $\tilde{u}(\vecX)$
minimizes the interpolation error $\sigma^2_{H, \hRatio}(\tilde{u}, F, G, \rho)$
for observations of $u(\vecX)$ over $D_H$ and observations of $f(\vecX)$ over
$D_{\frac{H}{\hRatio}}$:
\begin{equation} \label{eq:error_vfgp}
    \sigma^2_{H, \hRatio}(\tilde{u}, F, G, \rho)
        = \frac{1}{\mu(\Omega_H)} \int_{\Omega_H}
                \mathbb{E} \left[\tilde{u}(\vecX) - u(\vecX) \right]^2
            d\vecX \,.
\end{equation}
% Recall that by $\sigma^2_{H}(\tilde{g}, G)$ we denote the interpolation error of the
% Gaussian process, observed at points from $D_H$, with the spectral density $G(\vecO)$.

\begin{theorem} \label{th:vfgp_error}
The minimum of interpolation error~\eqref{eq:error_vfgp} of the variable fidelity data model
$u(\vecX)$ from \eqref{eq:eqqq}, based on observations of $u(\vecX)$ at points from
$D_H$ and observations of $f(\vecX)$ at points from $D_{\frac{H}{\hRatio}}$, has the form:
\begin{equation} \label{eq:error_ready}
    \sigma^2_{H, \hRatio}(\tilde{u}, F, G, \rho)
        = \sigma^2_{H}(\tilde{g}, G) + \rho^2 \sigma^2_{\frac{H}{\hRatio}}(\tilde{f}, F)
        \,,
\end{equation}
where $\tilde{g}(\vecX)$ and $\tilde{f}(\vecX)$ minimize $\sigma^2_{H}(\tilde{g}, G)$ and $\sigma^2_{\frac{H}{\hRatio}}(\tilde{f}, F)$ respectively.
\end{theorem}

\subsection{Minimax Interpolation Error}

We obtain the minimax interpolation error for the variable fidelity case in the manner
similar to the single fidelity case. 
Let us assume that the true spectral densities
of the processes $f(\cdot)$ and $g(\cdot)$ are unknown, but sufficiently smooth, i.e.
they belong to classes $\mathcal{F}(L_f) = \mathcal{F}(L_f, \mathbf{1})$ and
$\mathcal{F}(L_g) = \mathcal{F}(L_g, \mathbf{1})$ respectively. 
Here for clarity of
the presentation we limit ourselves to the case $\vecL = \mathbf{1}\in\mathbb{R}^{\iD}$
and $H = h \mathrm{I}$ for some $h > 0$, where $\mathrm{I}\in\mathbb{R}^{\iD\times\iD}$.
In fact, results below hold in a more general setting, described in section~\ref{sec:single_fidelity_total}
and defined by general values of $\vecL \in\mathbb{R}^{\iD}$ and $H$. 
However, this
additional sophistication blurs the main conclusions and provides little additional
insight.

The goal is to obtain the minimax interpolation error for $u(\vecX)$.
In particular we want to get the minimax interpolation error for the variable fidelity
data
\begin{equation} \label{eq:risk_vfgp}
    R^{h, \hRatio} (L_f, L_g)
        = \inf_{f, g}
            \sup_{\substack{F \in \mathcal{F}(L_f), \\ G \in \mathcal{F}(L_g)}}
                \sigma^2_{h\mathrm{I}, \hRatio} (\tilde{u}, F, G, \rho) \,.
\end{equation}

% Direct application of Theorems~\ref{th:multivariateMinimax} and \ref{th:vfgp_error} leads to:
\begin{theorem} \label{th:multifidelityMinimax}
Minimax interpolation error~\eqref{eq:risk_vfgp} of 
model~\eqref{eq:eqqq}, based on observations of $u(\vecX)$ at points
from $D_H$ and observations of $f(\vecX)$ at points from $D_{\frac{H}{\hRatio}}$,
has the form
\begin{equation} \label{eq:minimax_error}
    R^{h, \hRatio} (L_f, L_g)
        = \rho^2 \frac{L_f}{2} \left(\frac{h}{\hRatio \pi} \right)^{2}
            + \frac{L_g}{2} \left(\frac{h}{\pi} \right)^{2} \,.
\end{equation}
\end{theorem}

\section{Optimal Ratio of sizes of Variable Fidelity Data Samples}
\label{sec:optimal_design}

Obtained results allow us to get the optimal ratio $\hRatio$ of sizes of variable
fidelity data samples. 
We consider the following setting: 
one evaluation of $u(\vecX)$ costs $c$, whereas one evaluation of $f(\vecX)$ is $1$; 
the total evaluation cost is equal to the
number of points in a unit hypercube $\frac{1}{h^{\iD}}$ multiplied by the evaluation
price; 
and the computational budget is set to $\Lambda$.

For such setup the total budget is equal to $c \frac{1}{h^{\iD}} + \sRatio \frac{1}{h^{\iD}}$,
where $\sRatio = \hRatio^{\iD}$ is the ratio of sizes of variable fidelity data
samples.

Using Theorem~\ref{th:multifidelityMinimax} we prove
\begin{theorem} \label{th:optimal_ratio_vfgp}
The minimum of the minimax interpolation error~\eqref{eq:minimax_error} given the
computational budget $\Lambda$ has the form
\begin{align*}
    &\min_{\substack{h, \sRatio: \\ \Lambda h^{\iD} = c + \sRatio}} R^{h, \hRatio} (L_f, L_g)    = \rho^2 \frac{L_f}{2} \left(\frac{c + \sRatio^*}{\pi \Lambda \sRatio^*}\right)^{\frac{2}{d}}
           + \frac{L_g}{2} \left(\frac{c + \sRatio^*}{\pi \Lambda}\right)^{\frac{2}{\iD}}
        \,,
\end{align*}
and the optimal ratio is $\sRatio^* = \left( \frac{L_f}{L_g} c \rho^2 \right)^{\frac{\iD}{\iD + 2}}$.
\end{theorem}
The optimal ratio $\sRatio^*$ depends on the relative cost $c$ of the high fidelity
function evaluation, the coefficient $\rho$ and the smoothnesses $L_f$ and $L_g$ of
$f(\vecX)$ and $g(\vecX)$ respectively.

If for the interpolation we use evaluations of $u(\vecX)$ exclusively,
then we get the following minimax interpolation error given the budget $\Lambda$:
\begin{equation*}
    \min_{\substack{h:\, \Lambda h^{\iD} = c}} R^{h} (L_f, L_g)
        =  \rho^2 \frac{L_f}{2} \left(\frac{c}{\pi \Lambda}\right)^{\frac{2}{d}}
                + \frac{L_g}{2} \left(\frac{c}{\pi \Lambda}\right)^{\frac{2}{d}}
        \,.
\end{equation*}
Note, that we can get similar results for a specific covariance function using
Theorem~\ref{th:vfgp_error} and Corollaries~\ref{col:exponential_error} and~\ref{col:squared_exponential_error}.

%!TEX root = article.tex

\begin{figure}[t]
    \centering
    \includegraphics[width=0.35\textwidth]{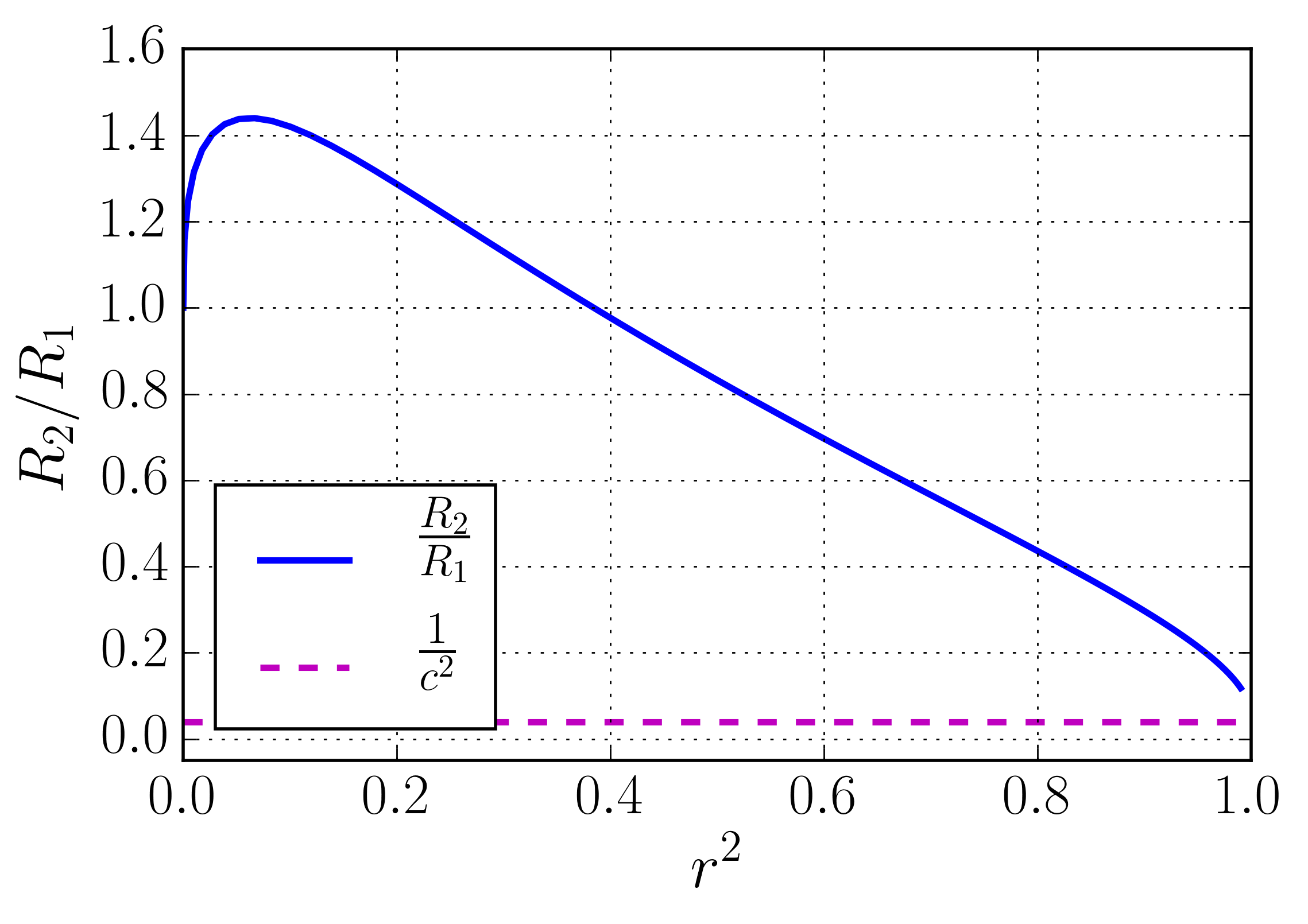}
    \caption{Dependence of the ratio $\frac{R_2}{R_1}$ of the minimax interpolation
    errors on the correlation coefficient $r$ for $L_f = 3, L_g = 1, c = 5$, $\iD = 1$.}
    \label{fig:plots-3-1-5}
\end{figure}

\subsection{Comparison of Minimax Interpolation Errors Under Different Scenarios}

Let us now investigate under what conditions and to what extent the usage of the
variable fidelity data can decrease the interpolation error compared to using single
fidelity data within the same computational budget. 
We denote by $R_2 = R^{h, \sRatio^*} (L_f, L_g, \rho)$
the minimax interpolation error, obtained when using the variable fidelity data,
and by $R_1 = R^{h} (L_f, L_g, \rho)$ the minimax interpolation error, obtained
when using only the high fidelity data. 
The ratio $\frac{R_2}{R_1}$ characterizes
benefits of the variable fidelity data over single fidelity data: $\frac{R_2}{R_1} \geq 1$
means there is no advantage to using the variable fidelity data, while $\frac{R_2}{R_1} < 1$
implies that the variable fidelity data improves the accuracy of the interpolation.

The ratio $\frac{R_2}{R_1}$ has the form:
\begin{equation*}
    \frac{R_2}{R_1}
        = \frac{\left(1
                    + \left(
                        \frac{L_f^{\iD} \rho^{2 \iD}}{L_g^{\iD} c^2} 
                    \right)^{\frac{1}{\iD + 2}} 
                \right)^{\frac{\iD + 2}{\iD}}}
               {1 + \rho^2 \frac{L_f}{L_g}} \,.
\end{equation*}

If we put $V_f = \mathbb{E} f^2(\vecX)$ and $V_g = \mathbb{E} g^2(\vecX)$, then the
correlation coefficient $r$ between $u(\vecX)$ and $f(\vecX)$ is 
$r = \frac{1}{\sqrt{1 + \frac{V_g}{V_f} \frac{1}{\rho^2} }}$. 
% Therefore, $r$ is a
% monotone function of the parameter $\rho$ for fixed $V_g$ and $V_f$. 
Thus for $r \rightarrow 0$
or $r \rightarrow 1$ it holds that
\begin{align*}
    r \rightarrow 0:\,\,
        &\frac{R_2}{R_1} \approx 1 + \frac{\iD + 2}{\iD} \left(\frac{L_f V_f}{L_g V_g} \right)^{\frac{\iD}{\iD + 2}}
                                     \frac{r^{\frac{2 \iD}{\iD + 2}}}{c^{\frac{2}{\iD + 2}}} \,, \\
    r \rightarrow 1:\,\,
        &\frac{R_2}{R_1} \approx \frac{1}{c^{\frac{2}{\iD}}}
                                 + \frac{2 + \iD}{\iD} \left(\frac{L_g V_f}{L_f V_g} \right)^{\frac{\iD}{\iD + 2}}
                                   \frac{(1 - r^2)^{\frac{\iD}{\iD + 2}}}{c^{\frac{4}{\iD(\iD + 2)}}} \,.
\end{align*}
If $r \rightarrow 0$ then the variable fidelity data is unable to improve the accuracy
of the interpolation, while when $r \rightarrow 1$ the ratio $\frac{R_2}{R_1}$ approaches
$\frac{1}{c^{\frac{2}{\iD}}}$, where usually $c \gg 1$. 
The speed of convergence improves
as $c$ increases and $\frac{L_g}{L_f}$ decreases, which means that if $g(\vecX)$ is
smoother than $f(\vecX)$, then the variable fidelity data improves accuracy of interpolation additionally. 
%more significantly.
%offers more accuracy benefits.
% Note also that results depend on input dimension of data $\iD$.

In figure~\ref{fig:plots-3-1-5} we show how the ratio $\frac{R_2}{R_1}$ depends on
$r = r(\rho)$ in case of $\iD = 1$. For small $r$ it holds that $R_2 > R_1$ no matter
how large our computational budget is, while for high enough $r$ the value of $\frac{R_2}{R_1}$
tends to $\frac{1}{c^2}$, $c \gg 1$.

% For $\iD = 1$ we can get the smallest $\rho$ that for fixed $c > 1$ provides smaller
% variable fidelity minimax interpolation error than that of single fidelity data, $R_2 \leq R_1$:
% \begin{equation*}
%     \rho = \left(\frac{L_g}{L_f} \right)^{\frac12}
%            \frac{\sqrt{3} \left(\frac{\sqrt{3}}{c} + \sqrt{4 - \frac{1}{c^2}}\right)^{\frac32}}
%                 {c^{\frac12}\left(1 - \frac{1}{c^2}\right)^{\frac32}} \,.
% \end{equation*}

Figure~\ref{fig:vf_usage_regions} depicts the smallest values of $r$ in the case of
$\iD = 1$, which for the fixed $c > 1$ provides $R_2 \leq R_1$.

\begin{figure}[t]
    \centering
    \includegraphics[width=0.35\textwidth]{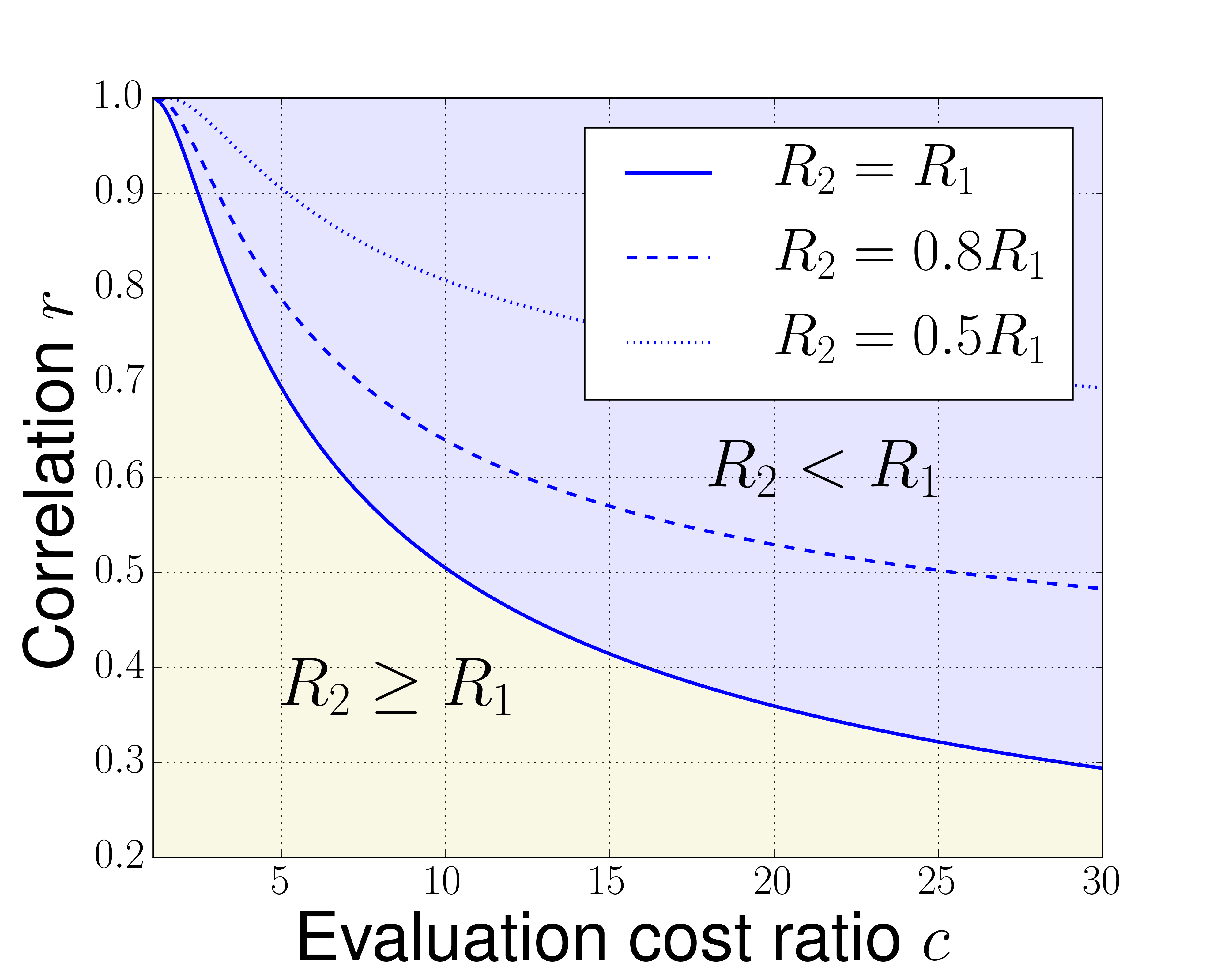}
    \caption{Curves $R_2 = k R_1$ for $L_f = 2$, $L_g = 1$, $\iD = 1$.}
    \label{fig:vf_usage_regions}
    % V_f = V_g = 1
\end{figure}

For $\iD > 1$ and $c\gg1$ the minimal value of $r$ that provides $R_2 \leq R_1$ is of
the order $\frac{1}{\sqrt{c}}$:
\begin{equation*}
    r \approx \frac{1}{\sqrt{c}}
              \left(\frac{V_f}{V_g} \right)^{\frac12}
              \left(\frac{L_g}{L_f} \right)^{\frac12} \,.
\end{equation*}

\subsection{Optimal Ratio of Sample Sizes for Variable Fidelity Data}

If we know the true covariance function it is easy to estimate the parameters $L_f$
and $L_g$ with the second derivatives of the covariance function $\frac{\partial^2 R(\vecX)}{\partial x_i \partial x_j}$
at the point $\vecX = \mathbf{0}$. 
However, in the small sample case it is difficult
to estimate the parameters of the covariance function~\cite{zaytsev2014properties}
or the sum of partial derivatives~\cite{kucherenko2009derivative} accurately.

Therefore, assuming $L_f = L_g$ and using Theorem~\ref{th:multifidelityMinimax}, we
propose {\bf Technique~\ref{alg:get_vf_design}}, that can be used to estimate the
optimal ratio of sample sizes $\sRatio^*$ and produce a design of experiments for
the case of variable fidelity data. 
The advantage of the proposed technique is that
it can be used even for a variable fidelity modeling approach different from the
Gaussian process regression framework; 
and it requires little prior knowledge about
the dependence structure between the high and the low fidelity functions, in particular,
we only have to estimate the correlation coefficient $r$.

\begin{algorithm}
  \begin{algorithmic}[1]
    \Require{Correlation $r$ between the variable fidelity observations,
    budget $\Lambda$, cost $c$ of one high fidelity function evaluation (the cost of evaluating
    the low fidelity function is fixed at $1$)}
    \Statex
    \Let{$\rho^2$}{$1 / (1 - \frac{1}{r^2})$}
    \Let{$\sRatio^*$}{$(c \rho^2)^{\frac{\iD}{\iD + 2}}$}
    \LetFour{$\sS_f$}{$\frac{\Lambda \sRatio^*}{c + \sRatio^*}$}{$\sS_u$}{$\frac{\Lambda}{c + \sRatio^*}$}
    \State Generate random nested designs of experiments $D_f, D_u$, $D_u \subseteq D_f$,
    with $|D_f| = n_f$, $|D_u| = n_u$.
    \State \Return{$D_f, D_u$}
  \end{algorithmic}
  \caption{Generation of designs of experiments $D_f$ and $D_u$ for evaluations of
  the low fidelity function and the high fidelity function respectively.}
  \label{alg:get_vf_design}
\end{algorithm}

%!TEX root = article.tex

\section{Experiments}
\label{expexpexp}

\begin{figure*}[ht!]
    \centering
    \begin{subfigure}[b]{0.3\textwidth}
        \includegraphics[width=\textwidth]{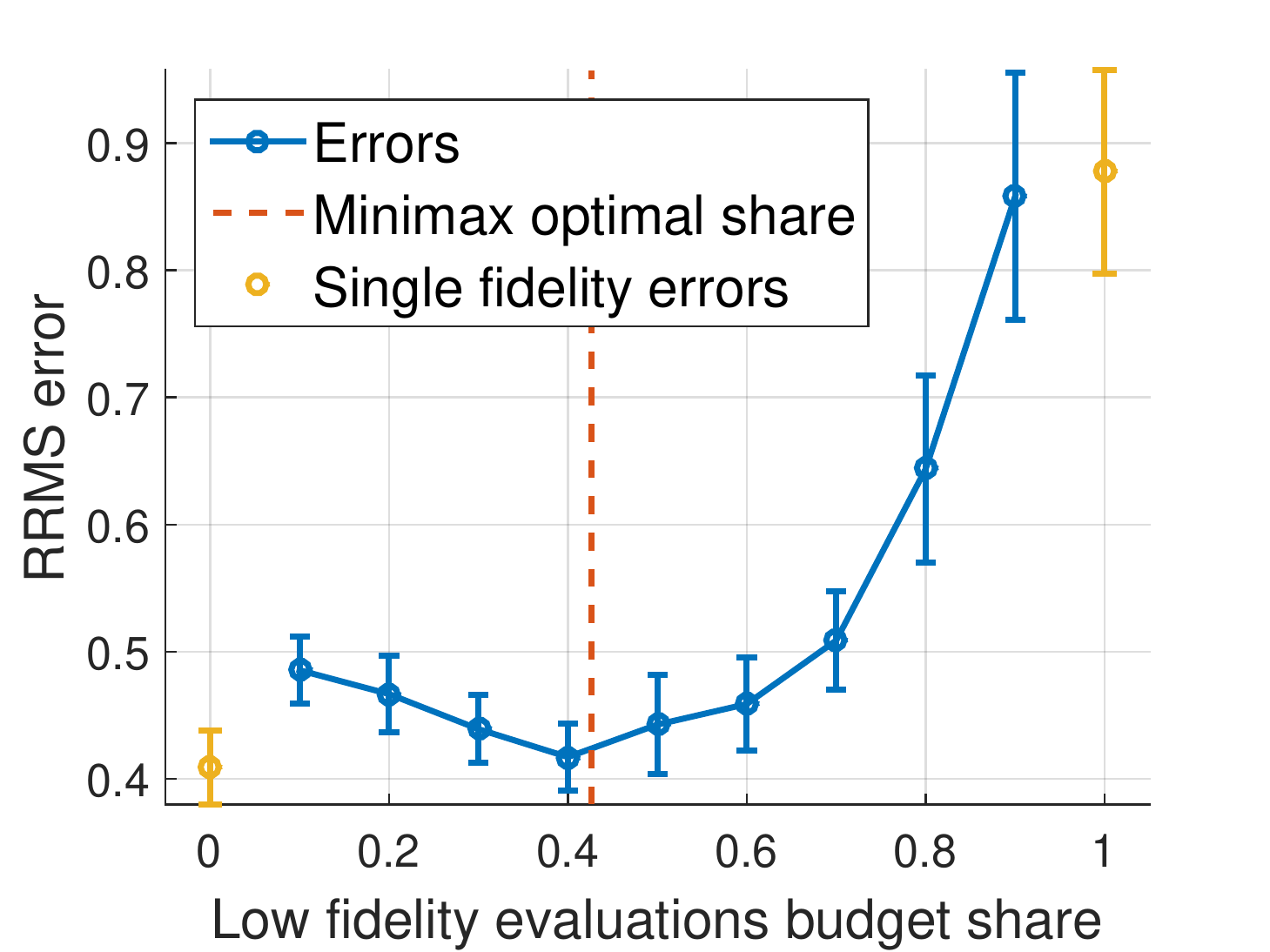}
        \caption{$c = 5, r = 0.8$}
    \end{subfigure}
    ~
    \begin{subfigure}[b]{0.3\textwidth}
        \includegraphics[width=\textwidth]{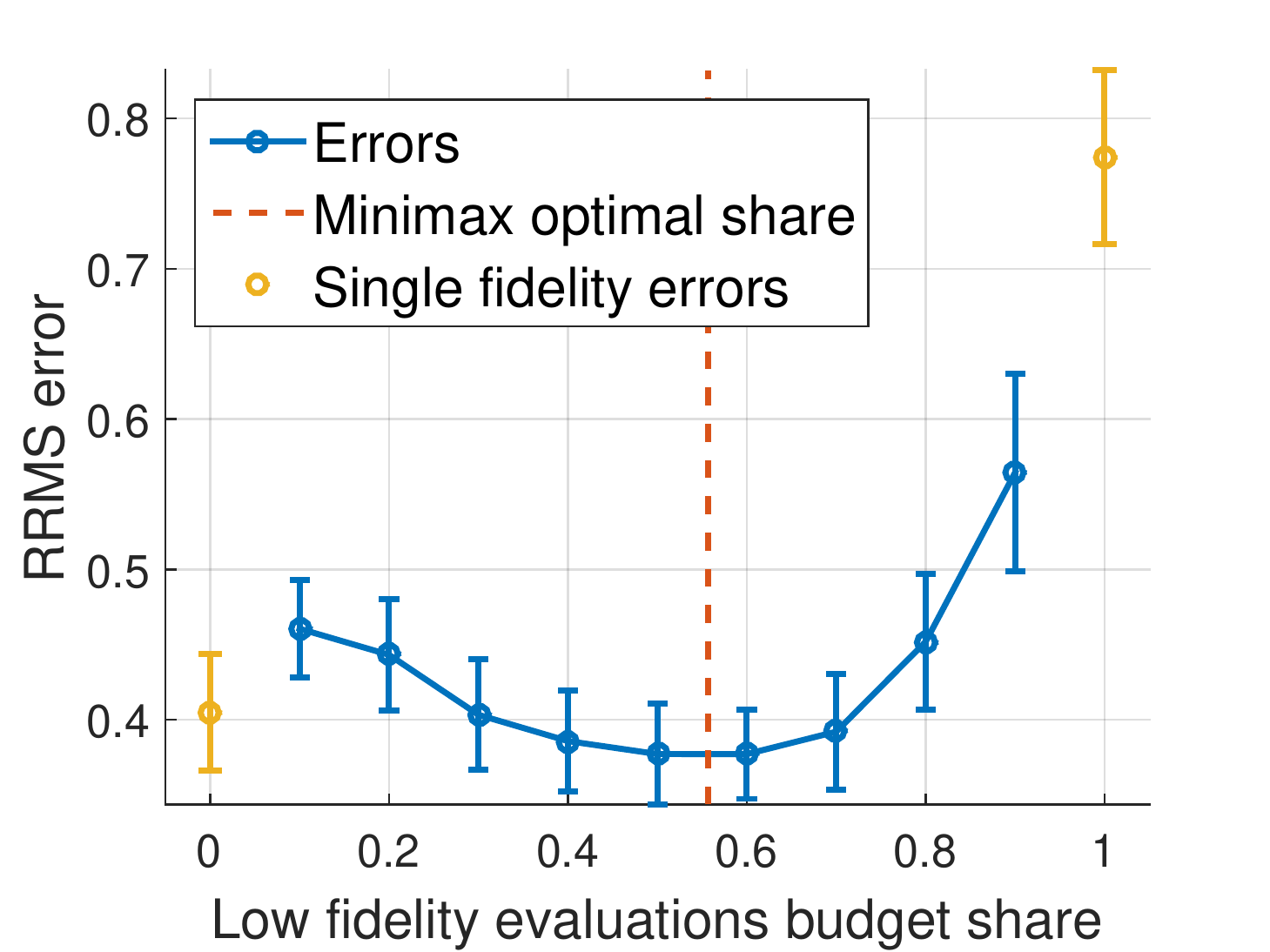}
        \caption{$c = 5, r = 0.9$}
    \end{subfigure}
    ~
    \begin{subfigure}[b]{0.3\textwidth}
        \includegraphics[width=\textwidth]{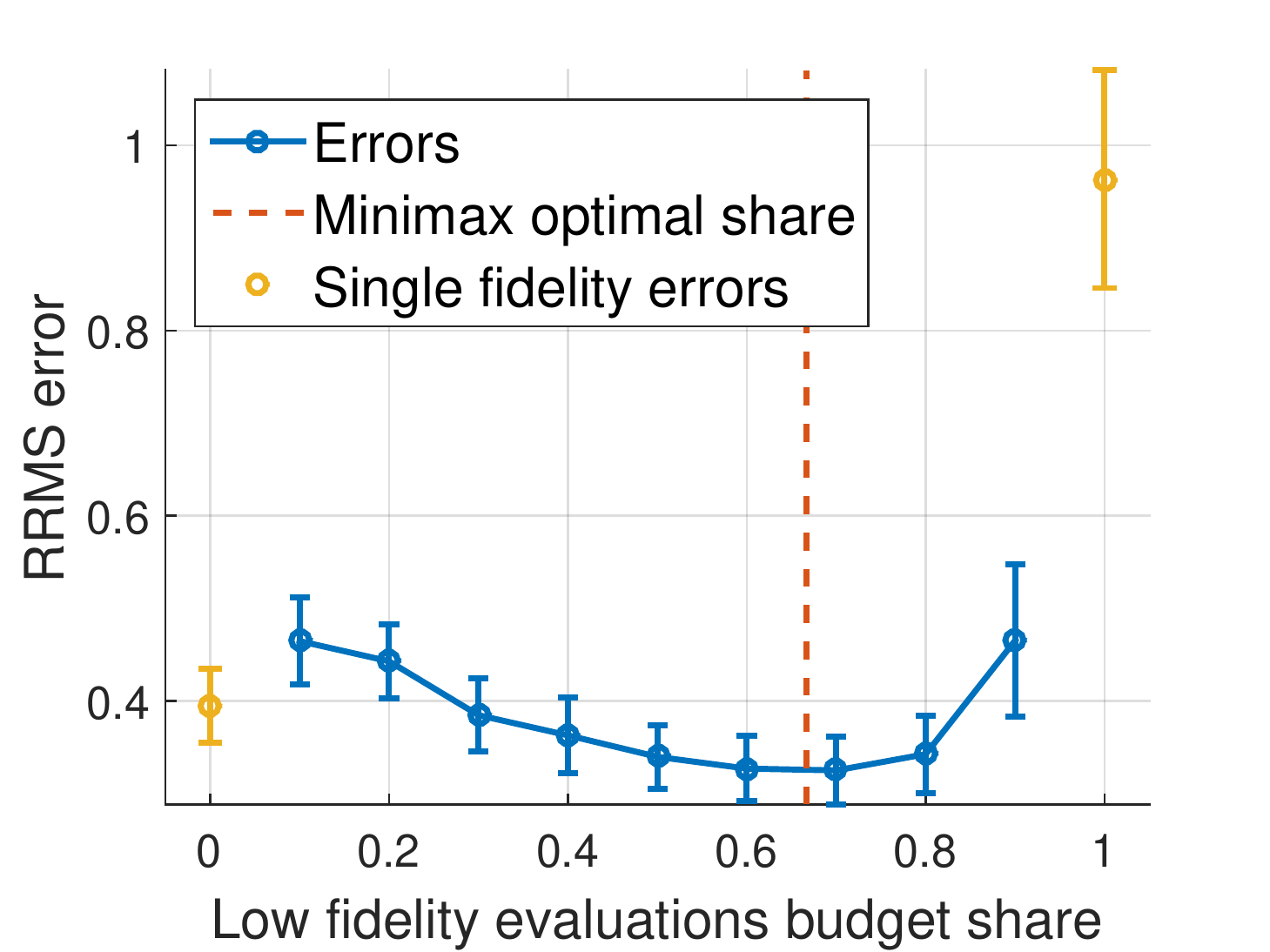}
        \caption{$c = 5, r = 0.95$}
    \end{subfigure}

    \begin{subfigure}[b]{0.3\textwidth}
        \includegraphics[width=\textwidth]{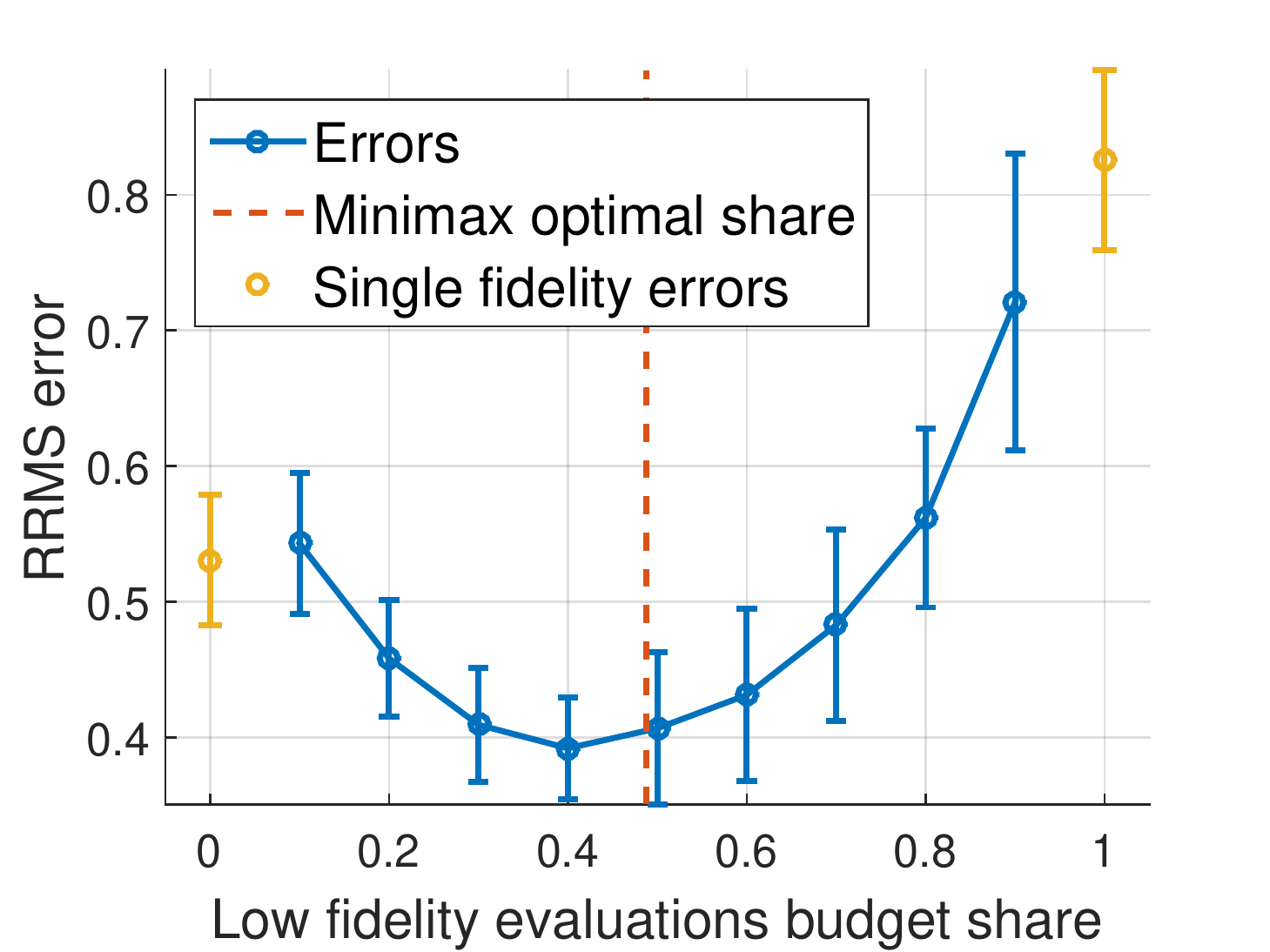}
        \caption{$c = 10, r = 0.8$}
    \end{subfigure}
    ~
    \begin{subfigure}[b]{0.3\textwidth}
        \includegraphics[width=\textwidth]{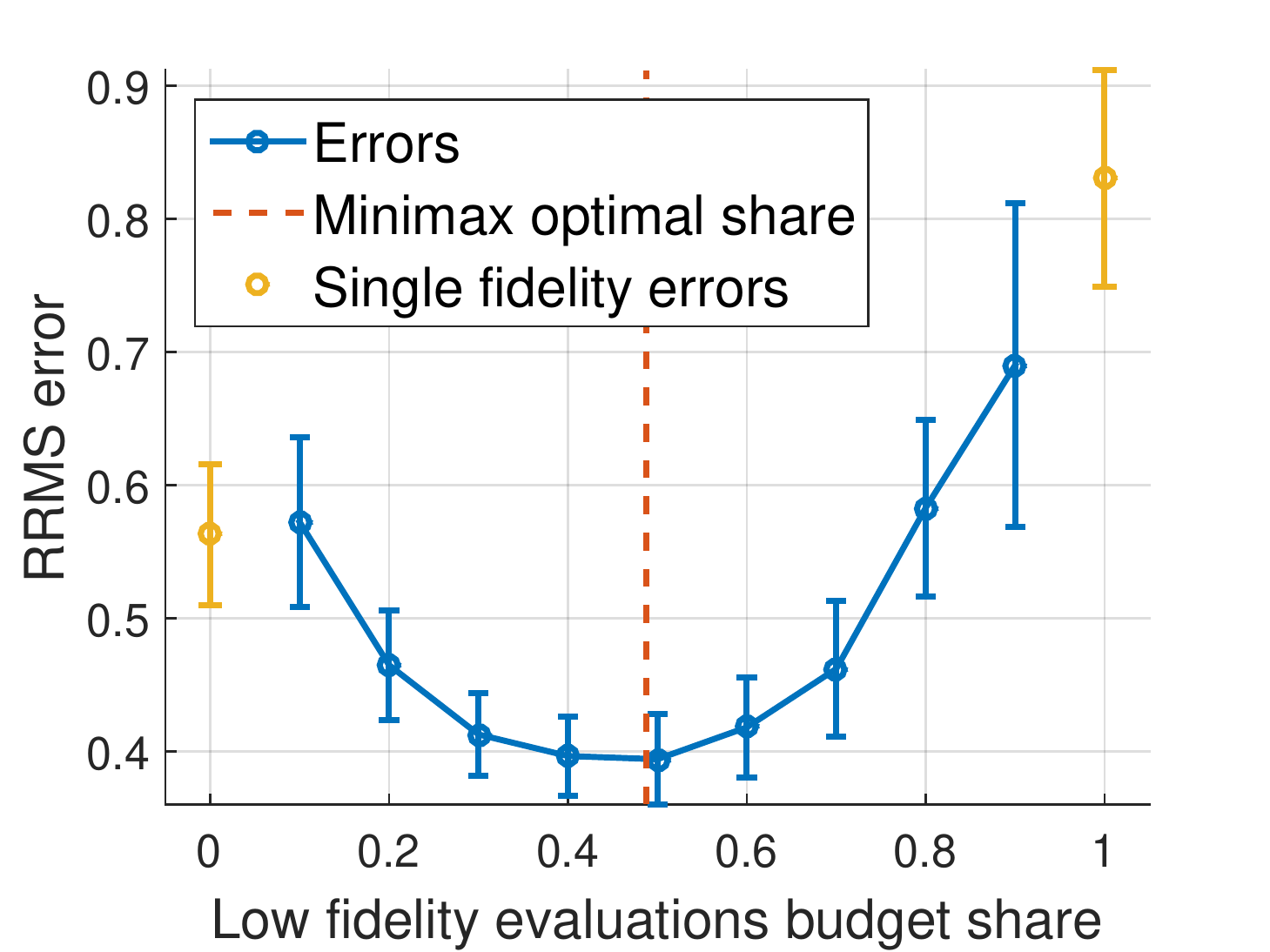}
        \caption{$c = 10, r = 0.9$}
    \end{subfigure}
    ~
    \begin{subfigure}[b]{0.3\textwidth}
        \includegraphics[width=\textwidth]{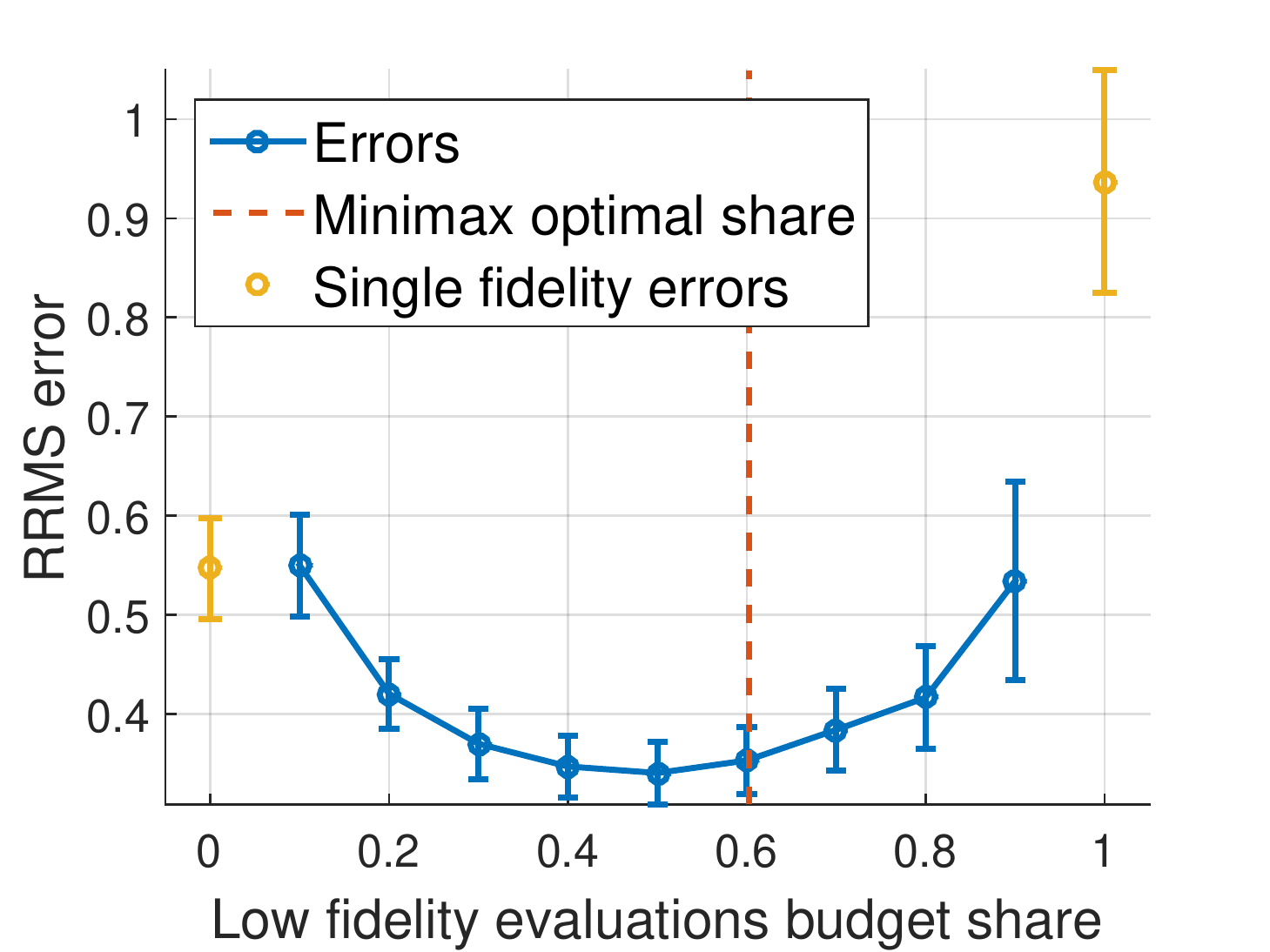}
        \caption{$c = 10, r = 0.95$}
    \end{subfigure}
    \caption{Synthetic data. Dependence of the RRMS error on the share of the budget,
    allocated for the low fidelity function evaluations. We consider the case of
    $\iD = 3$ and different correlations $r$ between the high and low fidelity functions.
    \textbf{\textcolor{gold}{Yellow}} points correspond to the case when we use either
    exclusively high or exclusively low fidelity data. The results are averaged across
    $20$ runs.}
    \label{fig:error-vs-sample-size}
\end{figure*}

We evaluate the performance of the proposed algorithm for estimation of the optimal
ratio of sample sizes for variable fidelity data in two steps: we start with synthetic
data generated as realizations of Gaussian processes, and then consider data that
originates from engineering applications.

We use the Mat\'ern covariance function $R_{\boldsymbol{\theta}}(\vecX - \vecX')$ with
$\nu=\frac32$ that provides differentiable realizations of Gaussian processes~\cite{rasmussen2006}:
\begin{equation*}
    R_{\boldsymbol{\theta}}(\vecX - \vecX')
        = (1 + \sqrt{3} \mathrm{d}_{\boldsymbol{\theta}}(\vecX - \vecX'))
            \exp (-\sqrt{3} \mathrm{d}_{\boldsymbol{\theta}}(\vecX - \vecX'))
        \,,
\end{equation*}
where $\mathrm{d}_{\boldsymbol{\theta}}(\vecX - \vecX') = \sqrt{\sum_{i = 1}^{\iD} \theta_i (x_i - x_i')^2}$. 
To construct a Gaussian process regression model we use Bayesian estimates of the
covariance function parameters~\cite{burnaev2016regression} obtained in a way similar
to~\cite{burnaev2015surrogate}, as open source software alternatives require manual
tuning for each particular problem to provide a reasonable comparison~\cite{le2014recursive}.

To assess model accuracy we use the Relative Root Mean Squared Error (RRMS) estimated
using a dedicated test sample in case of a synthetic data and the cross-validation
procedure in case of a real data. For a model $\tilde{u}(\vecX)$ and a test sample
$S_{*} = \{\vecX^{*}_i, u_{i}^{*} = u(\vecX^{*}_i)\}_{i = 1}^{\sS_t}$ the RRMS error
is given by
$
    \mathrm{RRMS} = \sqrt{\frac{\sum_{i = 1}^{\sS_t} (u^{*}_i - \tilde{u}(\vecX^{*}_i))^2}
                               {\sum_{i = 1}^{\sS_t} (u^{*}_i - \overline{u})^2}}
                    ,
$
where $\overline{u} = \frac{1}{\sS_t} \sum_{i = 1}^{\sS_t} u^{*}_i$.

%Additional results of experiments are provided in the supplementary materials.
Data and scripts, used to run the experiments, are at {\small \texttt{gitlab.com/JohnDoe1989/VariableFidelityData}}.

\subsection{Synthetic Data Experiments}

In this section we generate synthetic data as a realization of a Gaussian process
with a specified covariance function. We follow the model $u(\vecX) = \rho f(\vecX) + g(\vecX)$, 
with nested designs, i.e. $D_u \subseteq D_f$, and design points picked uniformly
at random from $[0, 1]^{\iD}$. 
The total computational budget is $300$, and the cost
of evaluating $u(\vecX)$ is either $c=5$ or $c=10$. 
Since the exact values of $\rho$ and $r$ are known, 
we use them in our experiments to calculate $\sRatio^*$.

Figures~\ref{fig:error-vs-sample-size} depict the dependence of the RRMS error on
the proportion of the computational budget allocated for the low fidelity function
evaluations. 
It can be seen that our estimate of the optimal ratio $\sRatio^*$ is
close to the true optimal ratio despite the fact that estimates of the unknown parameters
of the Gaussian Process regression model were used, 
and the design of experiments was not a grid.

\subsection{Baseline Techniques}

We compare our technique for estimation of the optimal ratio of sample sizes, which
we call {\bf MinMinimax}, to four different baseline heuristics:
\begin{itemize}
    \item {\bf High} --- only the high fidelity data is used,
    \item {\bf Low} --- we use only the low fidelity data,
    \item {\bf EqualSize} --- the sizes of low and high fidelity data samples are equal,
    \item {\bf EqualBudget} --- the budget is devoted equally to low
    and high fidelity function evaluations.
\end{itemize}

Relative sizes of samples for these techniques are given in Table~\ref{table:techniques_ratio}.

\begin{table}
    \centering
    \begin{tabular}{cll}
        \hline
        Baseline Technique & $n_{\mathrm{u}}$ & $n_{\mathrm{f}}$ \\
        \hline
        High & $\Lambda / c$ & 0 \\
        EqualSize & $\Lambda / (c + 1)$ & $\Lambda / (c + 1)$ \\
        EqualBudget & $\Lambda / (2 c)$ & $\Lambda / 2$ \\
        Low & 0 & $\Lambda$ \\
        MinMinimax & $\Lambda / (c + \sRatio^*)$ & $\sRatio^* \Lambda / (c + \sRatio^*)$ \\
        \hline
    \end{tabular}
    \caption{Sizes of the high fidelity sample ($n_u$) and the low fidelity sample
     ($n_f$) in case of the computational budget $\Lambda$.}
    \label{table:techniques_ratio}
\end{table}

\begin{table*}
    \centering
    \begin{tabular}{cccccccc}
        \hline
        Problem   & Output & Input & High & EqualSize & EqualBudget & MinMinimax & Low \\
                  & number & dimension \\  
        \hline
        Euler     & 1 & 11 & 0.7674 & 0.8925 & 0.8462 & {\bf 0.7420} & 0.9139 \\ 
        Euler     & 2 & 11  & {\bf 0.0668} & 0.0778 & 0.2699 & 0.3803 & 0.3974 \\   
        Airfoil   & 1 & 6  & 0.5462 & 0.5946 & 0.5390 & 0.5221 & {\bf 0.4852} \\ 
        Airfoil   & 2 & 6  & {\bf 0.1200} & 0.1422 & 0.1301 & 0.1381 & 0.2962 \\ 
        MachAngle & 1 & 2  & {\bf 0.0886} & 0.1064 & 0.1951 & 0.1951 & 0.4052 \\ 
        MachAngle & 2 & 2  & {\bf 0.0938} & 0.1148 & 0.1714 & 0.1796 & 0.3651 \\ 
        Press12   & 1 & 6  & 0.5599 & 0.6019 & 0.3580 & {\bf 0.2779} & 0.2843 \\ 
        Press12   & 2 & 6  & 0.4433 & 0.4918 & 0.2715 & {\bf 0.1768} & {\bf 0.1768} \\ 
        Press13   & 1 & 6  & 0.5596 & 0.5759 & 0.3861 & {\bf 0.3481} & 0.5435 \\ 
        Press13   & 2 & 6  & 0.4492 & 0.4852 & 0.2782 & {\bf 0.1798} & {\bf 0.1798} \\ 
        Disk      & 1 & 6  & 0.2999 & 0.3400 & 0.1922 & 0.1934 & {\bf 0.1638} \\
        Disk      & 2 & 6  & 0.4460 & 0.4570 & 0.2998 & 0.2998 & {\bf 0.2723} \\ 
        SVM       & 1 & 2  & {\bf 0.1487} & 0.1492 & 0.1849 & 0.1642 & 0.6081\\ 
        Supernova & 1 & 3  & 0.0395 & 0.0484 & {\bf 0.0180} & 0.0575 & 0.0575 \\ 
        \hline
    \end{tabular}
    \caption{RRMS errors averaged over $20$ runs of the cross-validation procedure
    for the problems with real data.}
    \label{table:techniques_rrms_cpress}
\end{table*}

% \begin{table*}
%     \centering
%     \begin{tabular}{cccccccc}
%         \hline
%         Problem   & Output & $\frac{R_2}{R_1}$ & High & EqualSize & EqualBudget & MinMinimax & Low \\
%         \hline
%         cPress12  & 1 & 0.6412 & 0.5599 & 0.6019 & 0.3580 & {\bf 0.2779} & 0.2843 \\ % cost ratio 5
%         cPress12  & 2 & 0.6314 & 0.4433 & 0.4918 & 0.2715 & {\bf 0.1768} & {\bf 0.1768} \\ % budget 300
%         cPress13  & 1 & 0.6884 & 0.5596 & 0.5759 & 0.3861 & {\bf 0.3481} & 0.5435 \\ % cost ratio 5
%         cPress13  & 2 & 0.6316 & 0.4492 & 0.4852 & 0.2782 & {\bf 0.1798} & {\bf 0.1798} \\ % budget 300
%         Disk      & 1 & 0.7164 & 0.2999 & 0.3400 & 0.1922 & 0.1934 & {\bf 0.1638} \\
%         Disk      & 2 & 0.7367 & 0.4460 & 0.4570 & 0.2998 & 0.2998 & {\bf 0.2723} \\ 
%         EulerFull & 1 & 1.0194 & 0.7674 & 0.8925 & 0.8462 & {\bf 0.7420} & 0.9139 \\ % budget 200
%         EulerFull & 2 & 0.9295 & {\bf 0.0668} & 0.0778 & 0.2699 & 0.3803 & 0.3974\\ % budget 200  
%         Airfoil   & 1 & 0.8144 & 0.7703 & 0.8202 & 0.7221 & {\bf 0.7018} & 2.3339 \\ % cost ratio 5
%         Airfoil   & 2 & 0.6448 & {\bf 0.1200} & 0.1422 & 0.1301 & 0.1381 & 0.2962 \\ % budget 300
%         MachAngle & 1 & 0.6536 & 0.3477 & 0.3479 & 0.3866 & 0.3866 & 8.9020 \\ % budget 100
%         MachAngle & 2 & 0.6052 & {\bf 0.0938} & 0.1148 & 0.1714 & 0.1796 & 0.3651 \\ % budget 100
%         \hline
%     \end{tabular}
%     \caption{RRMS errors averaged over $20$ runs of cross validation for real data problems}
%     \label{table:techniques_rrms_cpress}
% \end{table*}

\subsection{Real Data Experiments}

We consider the following real data problems.
The first three of them ({\bf Euler}, {\bf Airfoil}~\cite{bernstein2011comparison}, {\bf MachAngle}) are concerned with calculation of lift and drag coefficients of an airfoil depending on flight conditions and airfoil geometry.
To evaluate these outputs we use different solvers for the high and the low fidelity data sources. 
The next two problems ({\bf Press}~\cite{burnaev2015surrogate}, {\bf Disk}~\cite{zaytsev2016variable}) investigate dependence of maximum stress and maximum displacement on geometry of the equipment considered. 
Although three data fidelities are available in the {\bf Press} problem, in each experiment we use only two.
The last two problems (\cite{kandasamy2016gaussian}, {\bf SVM}, {\bf Supernova}) are related to modeling dependence of the goodness-of-fit on model parameters.
Details on the problems are provided in the supplementary materials.
Input dimensions for these problems are listed in Table~\ref{table:techniques_rrms_cpress}.

The budget $\Lambda$ is equal to $300$ for all problems except {\bf Euler}, 
as in this problem the sample size is small. 
For comparison consistency the cost ratio $c$ is kept at $5$ for all given problems. 
If the {\bf MinMinimax} technique returns the sample size $n_u < 1$, then only the low fidelity data is used.
For the {\bf MinMinimax} technique we use the correlation coefficient $r$ estimated using the whole available data sample.
In addition, to keep the comparison meaningful we normalize all the data before
constructing regression models to get variables with zero mean and unit variance.

{\bf Euler.}\!
Eleven input variables parametrize geometry of an airfoil.
We use two different solvers to obtain high and low fidelity values. 
We use an Euler equations solver to generate
the high fidelity data and a full potential equations solver to generate the low
fidelity data.

{\bf Airfoil.}\!
The geometry of an airfoil
and the flight regime (the speed and the angle of attack) are described by $52$ input
variables. We employ a dimension reduction procedure similar to the PCA, and model the
dependence on six input factors~\cite{bernstein2011comparison}. 
We use two different solvers to obtain high and low fidelity values.

{\bf MachAngle.}\!
Two input variables are the Mach number and the angle of attack for a specific airfoil. 
We use two different solvers to obtain high and low fidelity values. 
Low fidelity solver provides almost linear dependence.

{\bf Press.}
We model the maximum stress and the maximum displacement for a C-shaped press~\cite{burnaev2015surrogate}. 
Six input variables describe the geometry of the press,
and the fidelity of output depends on a mesh size. 
We generate three different data samples that correspond to high, 
moderate and low fidelity outputs. 
We refer to the case when we model the high fidelity output by $u(\vecX)$ and the moderate fidelity output by $f(\vecX)$ as \texttt{Press12}, 
and the case when we model the high fidelity
output by $u(\vecX)$ and the low fidelity output by $f(\vecX)$ as \texttt{Press13}.

{\bf Disk.}\!
We model the maximum stress and the maximum displacement of a rotating disk in an
engine~\cite{zaytsev2016variable}. 
Six input variables describe the geometry of the disk. 
We use two different solvers to obtain high and low fidelity values.

{\bf SVM.}\!
We model the dependence of the SVM classifier accuracy from the \textbf{sklearn}~\cite{scikit-learn} on the kernel bandwidth and the margin coefficient for the ``MAGIC Gamma
Telescope'' dataset~\cite{kandasamy2016gaussian}. 
We have two input variables.
As a measure of accuracy we use the area under the ROC curve as suggested by the
authors of the dataset~\cite{bock2004methods}. 
To generate the low fidelity dataset we estimate the accuracy
of the classifier constructed using $500$ training points, 
and to generate the high fidelity dataset 
we estimate the accuracy of the classifier constructed using $2000$ training points.

{\bf Supernova.}\!
We model the dependecy of the likelihood of the supernova redshift data on the three
fundamental physical constants, similarly to~\cite{kandasamy2016gaussian}.
To get a variable fidelity data we vary the grid size for a one-dimensional integration:
we generate the low fidelity data using the grid of size $3$ and the high fidelity
data using the grid of size $1000$. 
We note that if the size of the grid is greater than $3$, 
then the high and low fidelity functions become indistinguishable.

We provide errors in Table~\ref{table:techniques_rrms_cpress}, which show that the
best results are typically obtained using the proposed {\bf MinMinimax} approach.
However, there are two drawbacks: 
sometimes it is impossible to improve the model
accuracy using variable fidelity data; 
or too small sample size is selected 
making it impossible to construct a reliable regression model. 
For example, for the
{\bf Supernova} dataset the {\bf MinMinimax} method works poorly because it suggests
to use the high fidelity sample size equal to four, which is obviously insufficient
for the cokriging to work efficiently. 
That is why we suggest to impose a lower bound for the
size of the high fidelity data sample.

% In the {\bf MachAngle} dataset one of our assumptions is violated: we expect similar
% smoothness of the low and high fidelity data functions, i.e. $L_f = L_g$, whereas in
% the {\bf MachAngle} dataset the low fidelity function is linear, while the high fidelity
% function has significant nonlinearities.

\section{Conclusions}

We prove the minimax interpolation error for the Gaussian process regression in
the multivariate case. The obtained results are used to estimate the interpolation
error for the regression modeling with the variable fidelity data. This allows
us to identify settings in which the accuracy of the regression model can be improved
with the variable fidelity data.

Moreover, we estimate the optimal ratio of sizes of the variable fidelity data samples.
Using both synthetic and real problems, we demonstrate that this ratio can be used
when producing a design of experiments.

However, there is still room for improvement of the proposed approach: it requires
an accurate estimate of the correlation coefficient, and it doesn't take into account
inaccuracies of estimates of the regression model parameters. 
Furthermore, in this paper we consider the case of two fidelity levels only, 
whereas in practice multiple fidelity levels can be accessible.

\section*{Acknowledgments}
The research, presented in Section \ref{expexpexp} of this paper, was supported by the RFBR grants 16-01-00576 A and 16-29-09649 ofi\_m; the research, presented in other sections, was supported solely by the Russian Science Foundation grant (project 14-50-00150).

The authors would like to thank Yu. Golubev for his helpful comments, I. Nazarov and M. Panov for their help with proofreading.

\appendix
%!TEX root = article.tex

\section{PROOFS FOR SUBSECTION~\ref{sec:single_fidelity}}

\begin{proof}[Proof of Theorem~\ref{th:interpolation_error}]
It is easy to see that
\begin{align*}
\bbE [f(\vecX) - \tilde{f}(\vecX) ]^2 &= \int_{\bbR^{\iD}} F(\vecO) \left|1 - |H| \sum_{\vecK \in \bbZ^{\iD}} K(\vecX - \vecX_{\vecK}) \exp(-2\pi \mathrm{i} \vecO^{T} (\vecX_{\vecK} - \vecX)) \right|^2 d\vecO = \\
&= \int_{\bbR^{\iD}} F(\vecO) \left|1 - |H| \sum_{\vecK \in \bbZ^{\iD}} \left( \int_{\bbR^{\iD}} \hat{K}(\vecU) \exp(-2\pi \mathrm{i} \vecU^T (\vecX - \vecX_{\vecK})) d\vecU \right) \exp(-2\pi \mathrm{i} \vecO^{T} (\vecX_{\vecK} - \vecX)) \right|^2 d\vecO,
\end{align*}
where $\hat{K}(\vecU)$ is the Fourier transform of $K(\vecX)$.
As Poisson summation formula suggests:
\[
\sum_{\vecK \in \bbZ^{\iD}} \exp(2 \pi \mathrm{i} \vecK^T \vecO) = \sum_{\vecK \in \bbZ^{\iD}} \delta(\vecO + \vecK),
\]
where $\delta(\vecO)$ is the Dirac delta function,
then
\begin{align*}
&\bbE [f(\vecX) - \tilde{f}(\vecX) ]^2 = 
\int_{\bbR^{\iD}} F(\vecO) \left|1 - |H| \sum_{\vecK \in \bbZ^{\iD}} \int_{\bbR^{\iD}} \hat{K}(\vecU) \exp(2\pi \mathrm{i} (\vecO - \vecU)^T \vecX) \delta(\vecU - \vecO + H^{-1} \vecK)d\vecU \right|^2 d\vecO = \\
&= \int_{\bbR^{\iD}} \left| 1 - \sum_{\vecK \in \bbR^{\iD}} \hat{K} 
(\vecO - H^{-1} \vecK) \exp(2 \pi \mathrm{i} H^{-1} \vecX^T \vecK) \right|^2 d\vecO.
\end{align*}

Taking into account orthogonality of the system of functions $\exp(2 \pi \mathrm{i} H^{-1} \vecX^T \vecK)$ on $\vecX \in [0, h_1] \times \ldots \times [0, h_{\iD}]$ we integrate the equality to get the interpolation error
\[
\sigma^2_H(\tilde{f}, F) = 
\int_{\bbR^{\iD}} F(\vecO) \left|[1 - \hat{K}(\vecO)]^2 + \sum_{\vecK \in \bbZ^{\iD} \setminus \{\mathbf{0}\}} \hat{K}^2 (\vecO + H^{-1} \vecK) \right|^2 d\vecO.
\]

To get $\hat{K}(\vecO)$ that minimizes the interpolation error we rewrite it as
\[
\sigma^2_H(\tilde{f}, F) = 
\int_{\bbR^{\iD}} \left|[1 - \hat{K}(\vecO)]^2 F(\vecO) + \hat{K}(\vecO)^2 \sum_{\vecK \in \bbZ^{\iD} \setminus \{\mathbf{0}\}} \hat{F} (\vecO + H^{-1} \vecK) \right|^2 d\vecO.
\]
To minimize this error we solve this quadratic optimization problem for each $\vecO$ and get:
\[
\hat{K}(\vecO) = \frac{\hat{F} (\vecO)}
                      {\sum_{\vecK \in \bbZ^{\iD}} \hat{F} (\vecO + H^{-1} \vecK)}.
\]
Then 
\begin{equation}
\label{eq:sigma_formula}
\sigma^2_H(\tilde{f}, F) = \int_{\bbR^{\iD}} F(\vecO) \frac{\sum_{\vecK \in \bbZ^{\iD} \setminus \{\mathbf{0}\}} \hat{F} (\vecO + H^{-1} \vecK)}
                      {\sum_{\vecK \in \bbZ^{\iD}} \hat{F} (\vecO + H^{-1} \vecK)} d\vecO.
\end{equation}
\end{proof}

\begin{proof}[Proof of Remark \ref{lemma:best_approximation}]
It holds that the best approximation has the form
\[
\tilde{f}(\vecX) = \mu(\Omega_H) \sum_{\vecK \in \bbZ^{\iD}} \phi(\vecX, \vecX_{\vecK}) f(\vecX_{\vecK})
\] 
for some $\phi(\vecX, \vecX')$. As Wiener-Hopf equations for the covariance function $R(\vecX)$ hold, then
\begin{equation}
\label{eq:wiener}
\sum_{\vecK \in \bbZ^{\iD}} \phi(\vecX, \vecX_{\vecK}) R(\vecX_{\vecK} - \vecX_{\vecM}) = R(\vecX - \vecX_{\vecM})
\end{equation}
for all $\vecM \in \bbZ^{\iD}$.
Let us prove that $\phi(\vecX, \vecX_{\vecK}) = \phi(\vecX - \vecX_{\vecK})$.

Let us consider two sums from~\eqref{eq:wiener}:
\begin{align*}
\sum_{\vecK \in \bbZ^{\iD}} \phi(\vecX, \vecX_{\vecK}) R(\vecX_{\vecK} - \vecX_{\vecM}) &= R(\vecX - \vecX_{\vecM}), \\
\sum_{\vecK \in \bbZ^{\iD}} \phi(\vecX - \vecX_{\vecS}, \vecX_{\vecK}) R(\vecX_{\vecK} - \vecX_{\vecM - \vecS}) &= R(\vecX - \vecX_{\vecS} - \vecX_{\vecM - \vecS}).
\end{align*}
As $\vecX_{\vecM - \vecS} = H \vecM - H \vecS = \vecX_{\vecM} - \vecX_{\vecS}$, then
\[
\sum_{\vecK \in \bbZ^{\iD}} \phi(\vecX, \vecX_{\vecK}) R(\vecX_{\vecK} - \vecX_{\vecM}) =
R(\vecX - \vecX_{\vecM}) = R(\vecX - \vecX_{\vecS} - \vecX_{\vecM - \vecS}) = 
\sum_{\vecK \in \bbZ^{\iD}} \phi(\vecX - \vecX_{\vecS}, \vecX_{\vecK}) R(\vecX_{\vecK} - \vecX_{\vecM - \vecS}).
\]

Consequently, 
\[
\sum_{\vecK \in \bbZ^{\iD}} \left[\phi(\vecX, \vecX_{\vecK}) - \phi(\vecX - \vecX_{\vecS}, \vecX_{\vecK} - \vecX_{\vecS}) \right] R(\vecX_{\vecK} - \vecX_{\vecM}) = 0.
\]

Positive definiteness of the covariance function $R(\vecX)$ implies that
\[
\phi(\vecX, \vecX_{\vecK}) = \phi(\vecX - \vecX_{\vecS}, \vecX_{\vecK} - \vecX_{\vecS}).
\] 
For $\vecX_{\vecS} = \vecX_{\vecK}$ we get 
\[
\phi(\vecX, \vecX_{\vecK}) = \phi(\vecX - \vecX_{\vecK}, \mathbf{0}) = K(\vecX - \vecX_{\vecK}). 
\]

Due to Poisson summation formula it holds that
\[
\frac{1}{\mu(\Omega_H)} \Phi(\vecO) \sum_{\vecK \in \bbZ^{\iD}} F(\vecO - \vecO_{\vecK}) = F(\vecO),
\]
where $\Phi(\vecO)$ is the Fourier transform of $\phi(\vecX)$.
Then
\[
\Phi(\vecO) = \frac{\mu(\Omega_H) F(\vecO)}{\sum_{\vecK \in \bbZ^{\iD}} F(\vecO - \vecO_{\vecK})}.
\]
So, optimal interpolation has the form:
\[
\tilde{f}(\vecX) = \mu(\Omega_H) \sum_{\vecK \in \bbZ^{\iD}} K(\vecX - \vecX_{\vecK}) f(\vecX_{\vecK}).
\]
Also
\[
\hat{K}(\vecO) = \frac{\Phi(\vecO)}{\mu(\Omega_H)}.
\]
\end{proof}

\begin{proof}[Proof of Corollary~\ref{col:exponential_error}]
We get the interpolation error for an exponential covariance function of the form $R(x) = \sqrt{\frac{\pi}{2}} \exp \left(-\theta|x| \right)$
for $x \in \bbR$.
The spectral density for this covariance function is $F(\omega) = \frac{\theta}{\theta^2 + \omega^2}$.
% Note, that for realizations of Gaussian processes with this covariance function are not differentiable.

We want to evaluate the interpolation error
\[
\sigma_h^2 \left(\tilde{f}, F\right) = \int_{-\infty}^{\infty} F(\omega) \frac{\sum_{k \ne 0 } F(\omega + \frac{k}{h})}{\sum_k F(\omega + \frac{k}{h})} d\omega.
\]

It holds that
\begin{align*}
&\sum_{k = -\infty}^{\infty} F(\omega + \frac{k}{h}) = \sum_{k = -\infty}^{\infty} \frac{\theta}{(\omega + \frac{k}{h})^2 + \theta^2} = 
h \sum_{k = -\infty}^{\infty} \frac{h \theta}{(h \omega + k)^2 + h^2 \theta^2} = \\
&=\pi h \coth(\pi \theta h) \frac{1}{1 + \sin^2(\pi h \omega) (\coth^2(\pi \theta h) - 1)}.
\end{align*}

Then 
\[
\int_{-\infty}^{\infty} F(\omega) \frac{\sum_{k \ne 0 } F(\omega + \frac{k}{h})}{\sum_k F(\omega + \frac{k}{h})} d\omega = \int_{-\infty}^{\infty} \frac{\theta}{\theta^2 + \omega^2} \left(1 - \frac{\theta}{\theta^2 + \omega^2} \frac{1 + \sin^2(\pi h \omega) (\coth^2(\pi \theta h) - 1)}{\pi h \coth(\pi \theta h)} \right) d\omega.
\]

We can integrate three terms inside the integral analytically.
Namely,
\[
\int_{-\infty}^{\infty} \frac{\theta}{\theta^2 + \omega^2} d\omega = \pi.
\]
Also
\[
\int_{-\infty}^{\infty} \frac{\theta^2}{(\theta^2 + \omega^2)^2} d\omega = \frac{\pi}{2 \theta}.
\]
Finally
\[
\int_{-\infty}^{\infty} \frac{\theta^2}{(\theta^2 + \omega^2)^2} \sin^2(\pi \omega h) d\omega = -\frac{\pi^2 h}{2} \left(\cosh(\pi \theta h) - \sinh(\pi \theta h) \right) \left(\cosh(\pi \theta h) - \left(\frac{1}{\pi \theta h} + 1\right) \sinh(\pi \theta h) \right).
\]

Consequently, 
\begin{align*}
&\int_{-\infty}^{\infty} \frac{\theta}{\theta^2 + \omega^2} \left(1 - \frac{\theta}{\theta^2 + \omega^2} \frac{1 + \sin^2(\pi h \omega) (\coth^2(\pi \theta h) - 1)}{\pi h \coth(\pi \theta h)} \right) d\omega = \\
&\pi - \frac{\pi}{2 \pi \theta h \coth(\pi \theta h)} + \\
&+\frac{\pi^2}{2} \exp(- \pi \theta h) \left(\exp(-\pi \theta h) - \frac{1}{\pi \theta h} \sinh(\pi \theta h) \right) \frac{\coth^2(\pi \theta h) - 1}{\coth(\pi \theta h)}.
\end{align*}

For $h \rightarrow 0$ we get Taylor series for the obtained interpolation error:
\[
\sigma_h^2(\tilde{f}, F) = \frac{2 \pi^2}{3} \theta h + O((\theta h)^2).
\]
\end{proof}

\begin{proof}[Proof of Corollary~\ref{col:squared_exponential_error}]
Note, that the interpolation error has the form
\[
\sigma^2_{h}(\tilde{f}, F) = \int_{-\infty}^{\infty} F(\omega) \frac{\sum_{k \ne 0} F(\omega + \frac{k}{h})}{\sum_{s} F(\omega + \frac{s}{h})} d \omega.
\]
We get lower and upper bounds for this expression.
We denote $v = \frac{1}{h}$.

We get upper bound for the interpolation error by splitting integration region $(-\infty, \infty)$ 
to three intervals $(-\infty, -v / 2]$, $(-v / 2, v / 2]$, $(v / 2, +\infty)$ and obtaining an upper bound for each of them.

Note that 
\[
0 \leq \frac{\sum_{k \ne 0} F(\omega + kv)}{\sum_{s} F(\omega + sv)} \leq 1.
\]
Consequently, using Chernov type bounds~\cite{chang2011chernoff} we get
\begin{align}
\label{eq:v2infty_upper_bound}
&\int_{v / 2}^{\infty} F(\omega) \frac{\sum_{k \ne 0} F(\omega + kv)}{\sum_{s} F(\omega + sv)} d \omega \leq 
\int_{v / 2}^{\infty} F(\omega) d \omega = \\ 
&=\int_{v / 2}^{\infty} \frac{1}{\sqrt{\theta}} \exp \left(-\frac{\omega^2}{2 \theta}\right) d \omega \leq \sqrt{2} \exp \left(-\frac{v^2}{8 \theta} \right). \nonumber
\end{align}
In a similar way get an upper bound for the interval~$(-\infty, -v / 2)$.

Now we get an estimate for the interval~$(-v / 2, v / 2)$.
We start with an upper bound and a lower bound for series $\sum_{s \ne 0} F(\omega + sv)$.
Spectral density for squared exponential covariance function decreases at $[0, +\infty)$ with respect to $\omega$.
Thus,
\[
\int_{\Delta + u}^{+\infty} F(x) dx \leq \sum_{s = 1}^{\infty} \Delta F(\Delta s + u) 
\leq \Delta F(s + u) + \int_{\Delta + u}^{+\infty} F(x) dx.
\]
Using~\cite{abramowitz1964handbook}, Formula 7.1.13, we get
for $\omega$ such that $|\omega| \leq \frac{v}{2}$:
\[
\frac{4 \sqrt{\theta}}{v + \sqrt{v^2 + 16 \theta}} 
\exp \left(-\frac{v^2}{8 \theta}\right)
\leq \int_{\frac{v}{2}}^{\infty} F(\omega) d\omega \leq 
\frac{4 \sqrt{\theta}}{v + \sqrt{v^2 + \frac{32}{\pi} \theta}} 
\exp \left(-\frac{v^2}{8 \theta}\right).
\]
And 
\[
v \sum_{k \in \bbZ^{+}} F(\omega + k v) \leq 
v F(\omega + v) + \int_{\frac{v}{2}}^{\infty} F(\omega) d\omega \leq
\frac{v}{\sqrt{\theta}} \exp \left(-\frac{v^2}{8 \theta} \right) + \frac{4 \sqrt{\theta}}{v + \sqrt{v^2 + \frac{32}{\pi} \theta}} 
\exp \left(-\frac{v^2}{8 \theta} \right).
\]

Now we are ready to get an upper bound for the integral over the interval~$(-v / 2, v / 2)$ for big enough $v$:
\begin{align*}
&\int_{-v / 2}^{v / 2} F(\omega) \frac{\sum_{k \ne 0} F(\omega + kv)}{\sum_{s} F(\omega + sv)} d \omega \leq \\
&\leq \int_{-v / 2}^{v / 2} F(\omega) \frac{F(\omega + v) + F(\omega - v) + \frac{4 \sqrt{\theta}}{v \left( v + \sqrt{v^2 + \frac{32}{\pi} \theta} \right)} 
\exp \left(-\frac{v^2}{8 \theta} \right)}{F(\omega) + F(\omega + v) + F(\omega - v) + \frac{4 \sqrt{\theta}}{v \left(v + \sqrt{v^2 + \frac{32}{\pi} \theta}\right)} 
\exp \left(-\frac{v^2}{8 \theta} \right)} d \omega \leq \\
& \leq \int_{-v / 2}^{v / 2} F(\omega) \frac{F(\omega + v) + F(\omega - v)}{F(\omega) + F(\omega + v) + F(\omega - v)} d \omega + 
\int_{-v / 2}^{v / 2} F(\omega) \frac{\frac{4 \sqrt{\theta}}{v \left(v + \sqrt{v^2 + \frac{32}{\pi} \theta}\right)} 
\exp \left(-\frac{v^2}{8 \theta} \right)}{F(\omega) + \frac{4 \sqrt{\theta}}{v \left(v + \sqrt{v^2 + \frac{32}{\pi} \theta} \right)} 
\exp \left(-\frac{v^2}{8 \theta} \right)} d\omega \leq \\
&\leq \int_{-v / 2}^{v / 2} F(\omega + v) + F(\omega - v) d \omega +
  \frac{4 \sqrt{\theta}}{v + \sqrt{v^2 + \frac{32}{\pi} \theta}} 
\exp \left(-\frac{v^2}{8 \theta} \right)\leq \\
&\leq \frac{12 \sqrt{\theta}}{v + \sqrt{v^2 + \frac{32}{\pi} \theta}} 
\exp \left(-\frac{v^2}{8 \theta} \right) \leq \frac{7 \sqrt{\theta}}{v} \exp\left(-\frac{v^2}{8 \theta}\right).
\end{align*}

It holds that 
\[
\frac{\sum_{k \ne 0} F(\omega + kv)}{\sum_{s} F(\omega + sv)} \geq
\frac{F(\omega + v) + F(\omega - v)}{F(\omega) + F(\omega + v) + F(\omega - v)}.
\]
For $\omega$ such that $|\omega| \leq \frac{v}{2}$ we get that:
\[
1 + \frac{F(\omega + v)}{F(\omega)} + \frac{F(\omega - v)}{F(\omega)} \leq 3.
\]

Then for sufficiently large $v$ the following lower bound holds: 
\begin{align*}
&\int_{-\infty}^{\infty} F(\omega) \frac{\sum_{k \ne 0} F(\omega + kv)}{\sum_{s} F(\omega + sv)} d \omega \geq \int_{-v / 2}^{v / 2} F(\omega) \frac{\sum_{k \ne 0} F(\omega + kv)}{\sum_{s} F(\omega + sv)} d \omega \geq \\ 
&\geq \int_{-\frac{v}{2}}^{\frac{v}{2}} F(\omega) \frac{F(\omega + v) + F(\omega - v)}{F(\omega) + F(\omega + v) + F(\omega - v)} d \omega
\geq 
\int_{-\frac{v}{2}}^{\frac{v}{2}} \frac{F(\omega + v) + F(\omega - v)}{3} d \omega = \\
&= \frac{2}{3} \int_{\frac{v}{2}}^{\frac{3v}{2}} F(\omega) d\omega \geq
\frac{2}{3} \left(\frac{4 \sqrt{\theta}}{v + \sqrt{v^2 + 16 \theta}} \exp\left(-\frac{v^2}{8 \theta}\right) - 
\frac{4 \sqrt{\theta}}{3 v + \sqrt{9 v^2 + \frac{32}{\pi} \theta}}\exp\left(-\frac{9 v^2}{8 \theta}\right)\right) \geq \\
& \geq \frac{4}{3} \frac{\sqrt{\theta}}{v} \exp\left(-\frac{v^2}{8 \theta}\right).
\end{align*}
\end{proof}

\section{PROOFS FOR SUBSECTION~\ref{subsec:minimax_inter_error}}

We need the following lemma to complete the proof of the main result
\begin{lemma}
\label{lemma:sum_bound}
Let $c \geq 0$ and $\omega \geq 0$ be such that $c^2 + \omega^2 \leq 1$, $c^2 + (1 - \omega^2) \leq 1$.
Then
\begin{equation}
\label{eq:lemma_sum_bound}
\left(1 - \sqrt{c^2 + \omega^2} \right)^2 + \left(1 - \sqrt{c^2 + (1 - \omega)^2} \right)^2 \leq
\left(1 - \sqrt{c^2} \right)^2 = (1 - c)^2.
\end{equation}
\end{lemma}

\begin{proof}
We start with a scheme of the proof.
We prove that for $\omega$, that maximizes left hand side of the inequality~\eqref{eq:lemma_sum_bound},
this inequality holds.
To prove this we show that for admissible $\omega \in [1 - \sqrt{1 - c^2}, \frac12]$ derivative of the left hand side with respect to $\omega$ is smaller than zero for all admissible  $c$, so $\omega = 1 - \sqrt{1 - c^2}$ provides maximum of the left hand side, and for such $\omega$ inequality holds.

Partial derivative of the left hand side with respect to $\omega$ is equal to
\begin{align*}
g(\omega, c) &= \frac{\partial}{\partial \omega} \left(\left(1 - \sqrt{c^2 + \omega^2} \right)^2 + \left(1 - \sqrt{c^2 + (1 - \omega)^2} \right)^2 \right) = \\
&=-2 \frac{\left(1 - \sqrt{c^2 + \omega^2} \right) \omega}{\sqrt{c^2 + \omega^2}} + 
 2 \frac{\left(1 - \sqrt{c^2 + (1 - \omega)^2} \right) (1 - \omega)}{\sqrt{c^2 + (1 - \omega)^2}} = \\
&= -2 \left(\frac{1}{\sqrt{c^2 + \omega^2}} - 1 \right) \omega
 +2 \left(\frac{1}{\sqrt{c^2 + (1 - \omega)^2}} - 1 \right) (1 - \omega).
\end{align*}

If $\omega = \frac{1}{2}$, then the partial derivative is zero.
We show that for such $\omega < \frac12$ that $c^2 + \omega^2 < 1$, 
$c^2 + (1 - \omega)^2 < 1$, it holds that $g(\omega, c) < 0$.
This fact means that the initial function decreases for $\omega \in [1 - \sqrt{1 - c^2}, \frac12]$.

We start with maximization of $g(\omega, c)$ with respect to $c$.
The function $g(\omega, c)$ attains maximum at the edge of admissibility region or in a local optimum
with respect to $c$. 
To find local optima we search for $c$, such that the partial derivative $g(\omega, c)$ with respect to $c$ is equal to zero:
\[
\frac{c (1 - \omega)}{((1 - \omega)^2 + c^2)^{\frac32}} - \frac{c \omega}{(\omega^2 + c^2)^{\frac32}} = 0.
\]

Consequently, 
\begin{equation}
\label{eq:omega_ratio}
\frac{1 - \omega}{\omega} = \frac{((1 - \omega)^2 + c^2)^{\frac32}}{(\omega^2 + c^2)^{\frac32}}.
\end{equation}

So,
\[
c^2 = \omega^{\frac23} (1 - \omega)^{\frac23} (\omega^{\frac23} + (1 - \omega)^{\frac23}).
\]

We show that this is a local maximum.
Namely, we prove that the second partial derivative of $g(\omega, c)$ with respect to $c$ is smaller than $0$:
\[
- \frac{(1 - \omega)}{((1 - \omega)^2 + c^2)^{\frac32}} + \frac{\omega}{(\omega^2 + c^2)^{\frac32}} 
+ \frac{3 c^2 (1 - \omega)}{((1 - \omega)^2 + c^2)^{\frac52}} - \frac{3 c^2 \omega}{(\omega^2 + c^2)^{\frac52}}
\leq 0.
\]

Or:
\[
\frac{\omega}{(\omega^2 + c^2)^{\frac52}} ((\omega^2 + c^2) - 3 c^2)
- \frac{(1 - \omega)}{((1 - \omega)^2 + c^2)^{\frac52}} (((1 - \omega)^2 + c^2) - 3 c^2) 
\leq 0.
\]

In a local optimum \eqref{eq:omega_ratio} holds, and we can rewrite inequality as:
\[
\frac{(1 - \omega)}{(\omega^2 + c^2) ((1 - \omega)^2 + c^2)^{\frac32}} (\omega^2 - 2 c^2)
- \frac{(1 - \omega)}{((1 - \omega)^2 + c^2)^{\frac52}} ((1 - \omega)^2 - 2 c^2)
\leq 0.
\]

Then, 
\[
\frac{(1 - \omega)}{((1 - \omega)^2 + c^2)^{\frac52} (\omega^2 + c^2)} \left( ((1 - \omega)^2 + c^2) (\omega^2 - 2 c^2) -  (\omega^2 + c^2) ((1 - \omega)^2 - 2 c^2)\right)
\leq 0.
\]

Due to constraints on values of $\omega$ this is equivalent to:
\[
((1 - \omega)^2 + c^2) (\omega^2 - 2 c^2) - (\omega^2 + c^2) ((1 - \omega)^2 - 2 c^2) \leq 0, 
\]
or
\[
2 c^2 \omega^2 - c^2 (1 - \omega)^2 - 2 c^2 (1 - \omega)^2 + c^2 \omega^2 \leq 0.
\]

This inequality holds, as $\omega \leq \frac12$ and $(1 - \omega)^2 \geq \omega^2$.

So, the extremum is a local minimum, and the function attains maximum values at the edges of the admissibility region. 
Namely, $c^2 = 1 - (1 - \omega)^2$ or $c^2 = 0$ provides maximum values.

For such values of $c$ the derivative is smaller than zero.
Using $c^2 = 1 - (1 - \omega)^2$ we get --- 
\[
-2 \left(\frac{1}{\sqrt{1 - (1 - \omega)^2 + \omega^2}} - 1 \right) \omega \leq 0. 
\]
In a similar way for $c^2 = 0$
\[
-2 (1 - \omega) + 2 \omega = 4 \omega - 2 \leq 0. 
\]

Consequently, the target function decreases with respect to $\omega$ at $[1 - \sqrt{1 - c^2}, \frac12]$,
and $\omega = \frac12$ provides a local minimum.
So, the local maximum for left hand side is at $\omega = 1 - \sqrt{1 - c^2}$.
It is easy to see that in this case the left side of~\eqref{eq:lemma_sum_bound} is not larger than the right side.
\end{proof}

Let us now prove the main theorem.

\begin{proof}[Proof of Theorem~\ref{th:multivariateMinimax}]
We provide upper and lower bounds for $R^H(L, \vecL)$ that are equal to~
$\frac{L}{2 \pi^2} \max_{i \in \{1, \ldots, d\}} \left(\frac{h_i}{\lambda_i} \right)^2$.
We start with a lower bound, and then continue with an upper bound.

We consider a functional
\[
\Phi(F, \hat{K}) = \int_{\mathbb{R}^d} F(\vecO) \left[(1 - \hat{K}(\vecO))^2 + \sum_{\vecX \in D_{H^{-1}} \setminus \{\mathbf{0}\}} \hat{K}^2(\vecO + \vecX) \right] d\vecO,
\]
that is equal to the interpolation error $\sigma^2_H(\tilde{f}, F)$ for 
\[
\tilde{f}(\vecX) = \mu(\Omega_H) \sum_{\vecX' \in D_H} K(\vecX - \vecX') f(\vecX'), 
\]
such that $\hat{K}(\vecO)$ is the Fourier transform of $K(\vecX)$.

The functional is linear in $F(\vecO)$ and quadratic in  $\hat{K}(\vecO)$,
and we search for a saddle point of the functional $R^H(L, \vecL)$ such that:
\[
R^H(L, \vecL) = \inf_{\tilde{f}} \sup_{F \in \mathcal{F}(L, \vecL)} \sigma^2_H(\tilde{f}, F)
= \sup_{F \in \mathcal{F}(L, \vecL)} \inf_{\tilde{f}} \sigma^2_H(\tilde{f}, F).
\]

It holds that~\eqref{eq:sigma_formula}
\[
\min_{\hat{K}} \Phi(F, \hat{K}) = \int_{\bbR^{\iD}} F(\vecO) \frac{\sum_{\vecX \in D_{H^{-1}} \setminus \{\mathbf{0}\}} F(\vecO + \vecX)}{\sum_{\vecX \in D_{H^{-1}}} F(\vecO + \vecX)} d\vecO.
\]

Let us consider a class of spectral densities $F_{\varepsilon}(\vecO)$:
\[
F_{\varepsilon}(\vecO) =
\begin{cases}
\frac{A_{\varepsilon}}{(2 \varepsilon)^d}, &\exists \mathbf{s} \in U_h: \|\vecO - \mathbf{s}\|_{\infty} \leq \varepsilon, \\ 
0, &\text{otherwise,}
\end{cases}
\]
here $U_h = \left\{\left(0, 0, \ldots, \frac{1}{2h_j}, \ldots, 0\right), \left(0, 0, \ldots, -\frac{1}{2h_j}, \ldots, 0\right) \right\}$, and an index $j$ is such that 
\[
j = \mathrm{arg} \max_{i \in \{1, \ldots, d\}} \left(\frac{h_i}{\lambda_i} \right)^2.
\]

Due to~\eqref{eq:Fset}
\[
(2 \pi)^2 \int_{\mathbb{R}^d} F(\vecO) \sum_{i = 1}^d \lambda_i^2 \omega_i^2 d\vecO \leq L, 
\]
and for $\varepsilon \rightarrow 0$
\[
A_{\varepsilon} \rightarrow \frac{L}{2 \pi^2} \left(\frac{h_j}{\lambda_j} \right)^2.
\]
Really, for $\varepsilon \rightarrow 0$:
\[
(2 \pi)^2 \int_{\mathbb{R}^d} F(\vecO) \sum_{i = 1}^d \lambda_i^2 \omega_i^2 d\vecO \rightarrow 2 (2 \pi)^2 \frac{A_{\varepsilon}}{(2 \varepsilon)^d} (2 \varepsilon)^d \left(\frac{\lambda_j}{h_j} \right)^2 = 2 A_{\varepsilon} \left(\frac{\pi \lambda_j}{h_j} \right)^2 = L. 
\]

Now for $\varepsilon \rightarrow 0$ it holds that 
\[
\min_{\hat{K}} \Phi(F_{\varepsilon}, \hat{K}) \rightarrow 2 \frac{1}{2} \frac{A_{\varepsilon}}{(2 \varepsilon)^d} (2 \varepsilon)^d = A_{\varepsilon} = \frac{L}{2 \pi^2} \left(\frac{h_j}{\lambda_j} \right)^2.
\]

Consequently, we get a lower bound that equals $\frac{L}{2 \pi^2} \left(\frac{h_j}{w_j} \right)^2$. 
Now we continue with a proof of the upper bound.

For any $\hat{K}(\vecO)$ it holds that
\begin{align*}
& R^H(L, \vecL) \leq \max_{F \in \mathcal{F}(L, \vecL)} \Phi(F, \hat{K}) \leq \\
&\leq L \left(\frac{1}{2 \pi} \right)^2 \max_{\vecO} \left\{\frac{1}{\sum_{i = 1}^d \lambda_i^2 \omega_i^2} \left[(1 - \hat{K}(\vecO))^2 + \sum_{\vecX \in D_{H^{-1}} \setminus \{\mathbf{0}\}} \hat{K}^2(\vecO + \vecX)\right] \right\}.
\end{align*}

Now let us consider
\[
\hat{K}(\vecO) = 
\begin{cases}
 1 - \| \vecO \|, \| \vecO \|^2 \leq 1,\\ 
 0, \text{otherwise}.
\end{cases}
\]

Now we prove that for such $\hat{K}(\vecO)$ it holds that
\begin{equation}
\label{eq:upper_bound_2}
\left[(1 - \hat{K}(\vecO))^2 + \sum_{\vecX \in \mathbb{Z}^{\iD} \setminus \{\mathbf{0}\}} \hat{K}^2(\vecO + \vecX)\right] \leq 2 \|\vecO\|^2.
\end{equation}

It holds that $(1 - \hat{K}(\vecO))^2 \leq \|\vecO\|^2$.
Now let us prove that 
\begin{equation}
\label{eq:sum_bound}
\sum_{\vecX \in \mathbb{Z}^d \setminus \{\mathbf{0}\}} \hat{K}^2(\vecO + \vecX) \leq \|\vecO\|^2.
\end{equation}

We use mathematical induction by $\iD$ for $\vecO$ such that $\|\vecO\|_{\infty} < 1$. 
We prove that for $\|\vecO\|_{\infty} < 1$ and $c^2 \geq 0$: 
\[
\sum_{\substack{\vecX \in \mathbb{Z}^d \setminus \{\mathbf{0}\}, \\ \|\vecO + \vecX\|^2 + c^2 \leq 1}} \left(1 - \sqrt{c^2 + \sum_{i = 1}^d(\omega_i + x_i)^2} \right)^2 \leq \sum_{\substack{i \in \{1, \cdots, d\}, \\ c^2 + (1 - \omega_i)^2 \leq 1}} \left(1 - \sqrt{c^2 + (1 - \omega_i)^2} \right)^2.
\]
For $\iD = 1$ the induction statement is trivial, as the right hand side and the left hand side coincide.
Suppose that for $(\iD - 1)$ the induction statement holds.
Now let us prove that the induction statement holds for $\iD$.

For $\vecO = (\omega_1, \omega_2, \ldots, \omega_d)$ such that $\|\vecO\|_{\infty} < 1$, $i$-th component of the vector $\vecO + \vecX, \vecX \in \mathbb{Z}^d \setminus \{\mathbf{0}\}$ such that $\|\vecO + \vecX\| \leq 1$ is either $\omega_i$ or $(1 - \omega_i)$.
Consequently, all such $\vecO + \vecX$ has the form $\vecS \vecO + (\mathbf{1} - \vecS) (\mathbf{1} - \vecO)$,
where $\vecS$ is a vector with all components of it belong to $\{0, 1\}$.

It holds that $(1 - \sqrt{c^2 + (1 - \omega_1)^2 + \omega_2^2 + \ldots + \omega_d^2})^2 \leq (1 - \sqrt{c^2 + (1 - \omega_1)^2})^2$, if $c^2 + (1 - \omega_1)^2 + \omega_2^2 + \ldots + \omega_d^2 \leq 1$.

Now let us consider all terms of the form $(1 - \sqrt{c^2 + (1 - \omega_1)^2 + \ldots})^2$ for which there exists $j \ne 1$, such that $(1 - \omega_j)^2$ is in the sum inside the squared root.
Due to the induction statement sum of these terms is bounded by: 
\[
\sum_{\substack{i \in \{2, \ldots, d\}, \\ c^2 + (1 - \omega_1)^2 + (1 - \omega_i)^2 \leq 1}} (1 - \sqrt{c^2 + (1 - \omega_1)^2 + (1 - \omega_i)^2})^2.
\]
In the same way we prove that the sum of terms $(1 - \sqrt{c^2 + \omega^2_1 + \ldots})^2$ with a term $(1 - \omega_j)^2$ inside the root is upper bounded by: 
\[
\sum_{\substack{i \in \{2, \ldots, d\}, \\ c^2 + \omega_1^2 + (1 - \omega_i)^2 \leq 1}}(1 - \sqrt{c^2 + \omega_1^2 + (1 - \omega_i)^2})^2.
\]

Using Lemma~\ref{lemma:sum_bound} for a pair of terms $(1 - \sqrt{c^2 + (1 - \omega_1)^2 + (1 - \omega_i)^2})^2 + (1 - \sqrt{c^2 + \omega_1^2 + (1 - \omega_i)^2})^2$  we get:
\begin{align*}
&(1 - \sqrt{c^2 + (1 - \omega_1)^2 + (1 - \omega_i)^2})^2 + (1 - \sqrt{c^2 + \omega_1^2 + (1 - \omega_i)^2})^2 \leq \\
&\leq (1 - \sqrt{c^2 + (1 - \omega_i)^2})^2.
\end{align*}
This upper bound also holds if there are no or only one term for $i$-th index.
Consequently, the induction statement holds: the target sum is bounded by $\sum_{i = 1}^d (1 - \sqrt{c^2 + (1 - \omega_i)^2})^2$.

Using $c^2 = 0$ we get~\eqref{eq:sum_bound}.

Now let us consider the case $\|\vecO\|_{\infty} \geq 1$.
We look at the case
$\vecO = \{\hat{\omega}_1 + 1, \omega_2, \ldots, \omega_d\}$,
moreover $\|(\hat{\omega}_1, \omega_2, \ldots, \omega_d)\|_{\infty} < 1$, and $\hat{\omega}_1 \geq 0$, $\omega_i \geq 0, i = \overline{2, d}$.
Then $\|\vecO\|^2 = 1 + 2 \hat{\omega}_1 + \hat{\omega}_1^2 + \sum_{i = 2}^d \omega_i^2$.
For vector $(\hat{\omega}_1, \omega_2, \ldots, \omega_d)$ we have the induction statement above \eqref{eq:sum_bound}.
For the initial vector $\vecO$ we have an additional term $\hat{K}^2((\hat{\omega}_1, \omega_2, \ldots, \omega_d))$ if Euclidian norm of such a vector is below or equal $1$ --- but this new term is smaller or equal to $1$, as otherwise this term is not in the sum.
So, the target estimate for $\vecO$ holds.
Other cases for $\|\vecO\|_{\infty} > 1$ are similar.
Consequently for all $\vecO$ the estimate~\eqref{eq:upper_bound_2} holds.

It holds that
\[
\max_{\vecO} \frac{\sum_{i = 1}^{\iD} \omega_i^2}{\sum_{i = 1}^d \left(\frac{\lambda_i}{h_i} \right)^2 \omega_i^2} = \max_{i \in \{1, \ldots, \iD\}} \left(\frac{h_i}{\lambda_i} \right)^2.
\]

Consequently, an upper bound for minimax interpolation error holds
\[
R^H(L, \vecL) \leq \frac{L}{2 \pi^2} \max_{i \in \{1, \ldots, \iD\}} \left(\frac{h_i}{\lambda_i} \right)^2.
\]
The upper bound coincides with the lower bound.
The theorem holds.
\end{proof}

\section{PROOFS FOR SUBSECTION~\ref{subsec:error_vfgp}}

\begin{proof}[Proof of Theorem~\ref{th:vfgp_error}]
% TODO rewrite for multidimensional case
For convenience we redefine all points that belong to $D_H$ as $D_H = \{\vecX_i\}$
and all points that belong to $D_{\frac{H}{\hRatio}}$ as $D_{\frac{H}{\hRatio}} = \{\tilde{\vecX}_j\}$.
Then for Gaussian process regression the best unbiased estimator is linear in known values:
\[
\tilde{u}(\vecX) = \sum_{i} k_i u(\vecX_i) + \sum_{j} \tilde{k}_j f(\tilde{\vecX}_j).
\]
for some $k_i, \tilde{k}_j$.
Our problem is then to find coefficients $k_i$, $\tilde{k}_j$ that minimize
$\bbE(u(\vecX) - \tilde{u}(\vecX))^2$.
Using independence of random processes $f(\vecX)$ and $g(\vecX)$ we get:
\begin{align*}
\bbE(u(\vecX) - \tilde{u}(\vecX))^2 &= \bbE \left[ \rho f(\vecX) + g(\vecX) - \sum_{i} k_i (\rho f(\vecX_i) + g(\vecX_i)) - \sum_{j} \tilde{k}_j f(\tilde{\vecX}_j) \right]^2 = \\
&= \bbE \left[ \rho f(\vecX) - \sum_{i} \rho k_i f(\vecX_i) - \sum_{j} \tilde{k}_j f(\tilde{\vecX}_j) \right]^2 + \bbE \left[ g(\vecX) - \sum_{i} k_i g(\vecX_i)\right]^2.
\end{align*}

For each $i$ there exists an index $j$ such $\vecX_i = \tilde{\vecX}_j$.
Denote
\[
\tilde{k}'_j = 
\begin{cases}
\frac{1}{\rho} \tilde{k}_j, & \forall i, \tilde{\vecX}_j \ne \vecX_i, \\
\frac{1}{\rho} \tilde{k}_j + k_i, & \exists i, \tilde{\vecX}_j = \vecX_i.
\end{cases}
\]
There is a one-to-one correspondence between $(\{k_i\}, \{\tilde{k}_j\})$ and $(\{k_i\}, \{\tilde{k}'_j\})$, so minimization of $\bbE(u(\vecX) - \tilde{u}(\vecX))^2$ with respect to $k_i, \tilde{k}_j$ is equivalent to minimization of this function with respect to $k_i, \tilde{k}'_j$.
Then
\begin{align*}
&\bbE\left[ \rho f(\vecX) - \sum_{i} k_i \rho f(\vecX_i) - \sum_{j} \tilde{k}_j f(\tilde{\vecX}_j)\right]^2 + \mathbb{E}\left[g(\vecX) - \sum_{i} k_i g(\vecX_i)\right]^2 = \\ =&\rho^2 \bbE\left[f(\vecX) - \sum_{j} \tilde{k}'_j f(\tilde{\vecX}_j)\right]^2 + \bbE\left[g(\vecX) - \sum_{i} k_i g(\vecX_i)\right]^2.
\end{align*}
For terms
$\bbE\left[f(\vecX) - \sum_{j} \tilde{k}'_j f(\tilde{\vecX}_j)\right]^2$ and 
$\bbE\left[g(\vecX) - \sum_{i} k_i g(\vecX_i)\right]^2$ minimization problems are equivalent to that of single fidelity data --- and the first term contains only coefficients $\tilde{k}'_j$, the second term contains only coefficients~$k_i$.

For $k_i$ and $\tilde{k}'_j$ that minimize interpolation error at point for the single fidelity scenario it holds that $\tilde{k}'_j = K_f(\vecX - \vecX_j)$, $k_i = K_g(\vecX - \vecX_i)$ for some symmetric kernels
$K_f(\vecX - \vecX_j)$, $K_g(\vecX - \vecX_i)$.

Now we continue proof for $f(\vecX)$ and $g(\vecX)$ in a way similar to the single fidelity case.
For $\mathbb{E}\left[g(\vecX) - \sum_{i} K_g(\vecX - \vecX_i) g(\vecX_i)\right]^2$ it holds that
\begin{align*}
&\frac{1}{|H|} \int_{\substack{x_i \in [0, h_i], \\ i=\overline{1, \iD}}} \mathbb{E}\left[g(\vecX) - \sum_{i} K_g(\vecX - \vecX_i) g(\vecX_i)\right]^2 d\vecX = \\
&= \int_{\bbR^{\iD}} G(\vecO) \left[ \left[1 - \hat{K}_{g}(\vecO) \right]^2 + \sum_{\vecK \in \bbZ^{\iD} \setminus  
\{\mathbf{0}\}} \hat{K}_g^2 \left(\vecO + H^{-1} \vecK \right)\right] d\vecO
\end{align*}

In a similar way we get for the interval $[0, \frac{h_1}{\hRatio}] \cdot \ldots \cdot [0, \frac{h_{\iD}}{\hRatio}]$ for $\bbE\left[f(\vecX) - \sum_{j} K_f(\vecX - \tilde{\vecX}_j) f(\tilde{\vecX}_j)\right]^2$:
\begin{align*}
&\frac{\hRatio^{\iD}}{|H|} \int_{\substack{x_i \in [0, \frac{h_i}{\hRatio}], \\ i=\overline{1, \iD}}} \bbE \left[f(\vecX) - \sum_{j} K_f(\vecX - \tilde{\vecX}_j) f(\tilde{\vecX}_j)\right]^2 d\vecX = \\ 
&= \int_{\bbR^{\iD}} F(\vecO) \left[ \left[1 - \hat{K}_{f}(\vecO) \right]^2 + \sum_{\vecK \in \bbZ^{\iD} \setminus  
\{\mathbf{0}\}} \hat{K}_f^2 \left(\vecO + H^{-1} \vecK \right)\right] d\vecO.
\end{align*}
Consequently, 
\begin{align*}
&\frac{1}{|H|} \int_{\substack{x_i \in [0, h_i], \\ i=\overline{1, \iD}}}  \bbE \left[f(\vecX) - \sum_{j} K_f(\vecX - \tilde{\vecX}_j) f(\tilde{\vecX}_j)\right]^2 d\vecX = \\
&= \int_{\bbR^{\iD}} F(\vecO) \left[ \left[1 - \hat{K}_{f}(\vecO) \right]^2 + \sum_{\vecK \in \bbZ^{\iD} \setminus 
\{\mathbf{0}\}} \hat{K}_f^2 \left(\vecO + \hRatio H^{-1} \vecK \right)\right] d\vecO.
\end{align*}

So, the target interpolation error~\eqref{eq:error_vfgp} has the form:
\begin{align*}
\sigma^2_{H, \hRatio}(\tilde{u}, F, G, \rho) &= \int_{\bbR^{\iD}} G(\vecO) \left[ \left[1 - \hat{K}_{g}(\vecO) \right]^2 + \sum_{\vecK \in \bbZ^{\iD} \setminus 
\{\mathbf{0}\}} \hat{K}_g^2 \left(\vecO + H^{-1} \vecK \right)\right] d\vecO + \\
&+ \rho^2 \int_{\bbR^{\iD}} F(\vecO) \left[ \left[1 - \hat{K}_{f}(\vecO) \right]^2 + \sum_{\vecK \in \bbZ^{\iD} \setminus  
\{\mathbf{0}\}} \hat{K}_f^2 \left(\vecO + \hRatio H^{-1} \vecK \right)\right] d\vecO.
\end{align*}

Finally,
\[
\sigma^2_{H, \hRatio}(\tilde{u}, F, G, \rho)  = 
\sigma^2_{H}(\tilde{g}, G) + \rho^2 \sigma^2_{\frac{H}{\hRatio}}(\tilde{f}, F).
\]
\end{proof}

\section{PROOFS FOR SECTION~\ref{sec:optimal_design}}

In this section we provide a proof of Theorem~\ref{th:optimal_ratio_vfgp}.

\begin{proof}[Proof of Theorem~\ref{th:optimal_ratio_vfgp}]
Minimax interpolation error has the form:
\[
R_2 = \frac{L_g}{2} \frac{1}{\pi^2} \left(\frac{c + (\hRatio^*)^{\iD}}{\Lambda} \right)^{\frac{2}{d}} + 
      \rho^2 \frac{L_f}{2} \frac{1}{\pi^2} \left(\frac{c + (\hRatio^*)^{\iD}}{(\hRatio^*)^{\iD} \Lambda} \right)^{\frac{2}{d}}.
\]

Denote $\sRatio = (\hRatio^*)^d$.
Then we need to minimize with respect to $a$ the following expression 
\[
\frac{L_g}{2} \left(c + \sRatio\right)^{\frac{2}{d}} + 
\rho^2 \frac{L_f}{2} \left(\frac{c + \sRatio}{\sRatio} \right)^{\frac{2}{d}}.
\]

Partial derivative with respect to $\sRatio$ should equal $0$:
\[
\frac{L_g}{2} \left(c + \sRatio\right)^{\frac{2}{d} - 1} \frac{2}{d} + 
\rho^2 \frac{L_f}{2} \left(\frac{c + \sRatio}{\sRatio} \right)^{\frac{2}{d} - 1} \frac{2}{d} \frac{-c}{\sRatio^2} = 0.
\]
Consequently
\[
\frac{L_g}{2} + 
\rho^2 \frac{L_f}{2} \sRatio^{1 - \frac{2}{d}} \frac{-c}{\sRatio^2} = 0.
\]
So
\[
L_g = L_f \frac{\rho^2 c}{\sRatio^{1 + \frac{2}{d}}}.
\]
Finally,
\[
\sRatio = \left(c \rho^2 \frac{L_f}{L_g} \right)^{\frac{d}{d + 2}},
\]
also we get that
\[
\hRatio^* = \sqrt[d + 2]{c \rho^2 \frac{L_f}{L_g}}.
\]
\end{proof}

\bibliographystyle{plain}
\bibliography{vfgp_problem,compare_vf_sizes} 

\end{document}